\documentclass[sigconf]{aamas}

\usepackage{balance} 
\usepackage{algorithm2e}
\usepackage{amsmath}
\usepackage{amsthm}
\usepackage{amsfonts}
\usepackage{empheq}
\usepackage{longtable,tabularx}
\usepackage{booktabs}
\usepackage{multirow}
\usepackage[table]{xcolor}
\usepackage{siunitx}
\usepackage{etoolbox}
\usepackage{cancel}
\usepackage{subcaption}
\usepackage{orcidlink}

\renewcommand\footnotetextcopyrightpermission[1]{}

\AtBeginEnvironment{equation}{\fontsize{10pt}{9pt}\selectfont}
\AtBeginEnvironment{equation*}{\fontsize{10pt}{9pt}\selectfont}
\everymath{\fontsize{10pt}{8pt}\selectfont}
\RestyleAlgo{ruled}

\newtheorem{lemma}{Lemma}
\newtheorem{proposition}{Proposition}
\newcommand{\ind}{\perp\!\!\!\!\perp}

\usepackage[dvipsnames]{xcolor}
\DeclareMathOperator{\EX}{\mathbb{E}}

\newcommand{\boxedeq}[2]{\begin{empheq}[box={\fboxsep=6pt\fbox}]{align}\label{#1}#2\end{empheq}}
\DeclareMathOperator*{\argmax}{argmax} 

\acmSubmissionID{929}

\title[]{A Semi-Decentralized Approach to Multiagent Control}

\author{Mahdi Al-Husseini \orcidlink{0009-0003-9087-814X}}
\affiliation{
  \institution{Stanford University}
  \city{Stanford}
  \country{United States}}
\email{mah9@stanford.edu}

\author{Mykel J. Kochenderfer \orcidlink{0000-0002-7238-9663}}
\affiliation{
  \institution{Stanford University}
  \city{Stanford}
  \country{United States}}
\email{mykel@stanford.edu}

\author{Kyle H. Wray \orcidlink{0000-0001-6986-9941}}
\affiliation{
  \institution{Northeastern University}
  \city{Boston}
  \country{United States}}
\email{k.wray@northeastern.edu}

\begin{abstract}
We introduce an expressive framework and algorithms for the semi-decentralized control of cooperative agents in environments with communication uncertainty. Whereas semi-Markov control admits a distribution over time for agent actions, semi-Markov communication, or what we refer to as semi-decentralization, gives a distribution over time for what actions and observations agents can store in their histories. We extend semi-decentralization to the partially observable Markov decision process (POMDP). The resulting SDec-POMDP unifies decentralized and multiagent POMDPs and several existing explicit communication mechanisms. We present recursive small-step semi-decentralized A* (RS-SDA*), an exact algorithm for generating optimal SDec-POMDP policies. RS-SDA* is evaluated on semi-decentralized versions of several standard benchmarks and a maritime medical evacuation scenario. This paper provides a well-defined theoretical foundation for exploring many classes of multiagent communication problems through the lens of semi-decentralization. 
\end{abstract}


\newcommand{\BibTeX}{\rm B\kern-.05em{\sc i\kern-.025em b}\kern-.08em\TeX}

\begin{document}


\pagestyle{fancy}
\fancyhead{}


\maketitle 


\section{Introduction}

Many complex real-world problems require the coordination of multiple cooperative agents to solve, but feature limited opportunities for information exchange. The decentralized partially observable Markov decision process (Dec-POMDP) formalizes multi-agent planning and control in settings where explicit communication is impossible \cite{bernstein2002complexity, becker2004solving, amato2013decentralized}. Several model variants take advantage of existing, if limited, information structures and extend the Dec-POMDP to explicitly incorporate costly \cite{goldman2008communication}, delayed \cite{oliehoek2012tree, oliehoek2013sufficient}, lossy \cite{tung2021effective}, or intermittent \cite{zhang2024multiagent} communication. 

In this paper we pursue a general framework that unifies several multiagent communication mechanisms. We are interested in problems whose communication dynamics are conditioned with some probability on the state, joint actions, or joint observations. An example domain is area-wide Global Positioning System (GPS) denial via jamming \cite{grant2009gps}. As seen in the evacuation scenario depicted in Figure \ref{fig:cover_photo}, agents need to coordinate joint tasks in environments with degraded, denied, or disrupted communication. Agents must reason about which actions to take in light of available communication, the influence actions taken have on future communication, and future communication's influence on future actions. In decentralized systems, information may be action-gated in costly communication, constrained by channel capacity in lossy communication, temporally offset in intermittent or delayed communication, and either deterministic or stochastic in nature, while semi-decentralized systems generalize all of these through probabilistic information flow.

\begin{figure}[t!]
  \centering
  \includegraphics[width=0.99\linewidth]{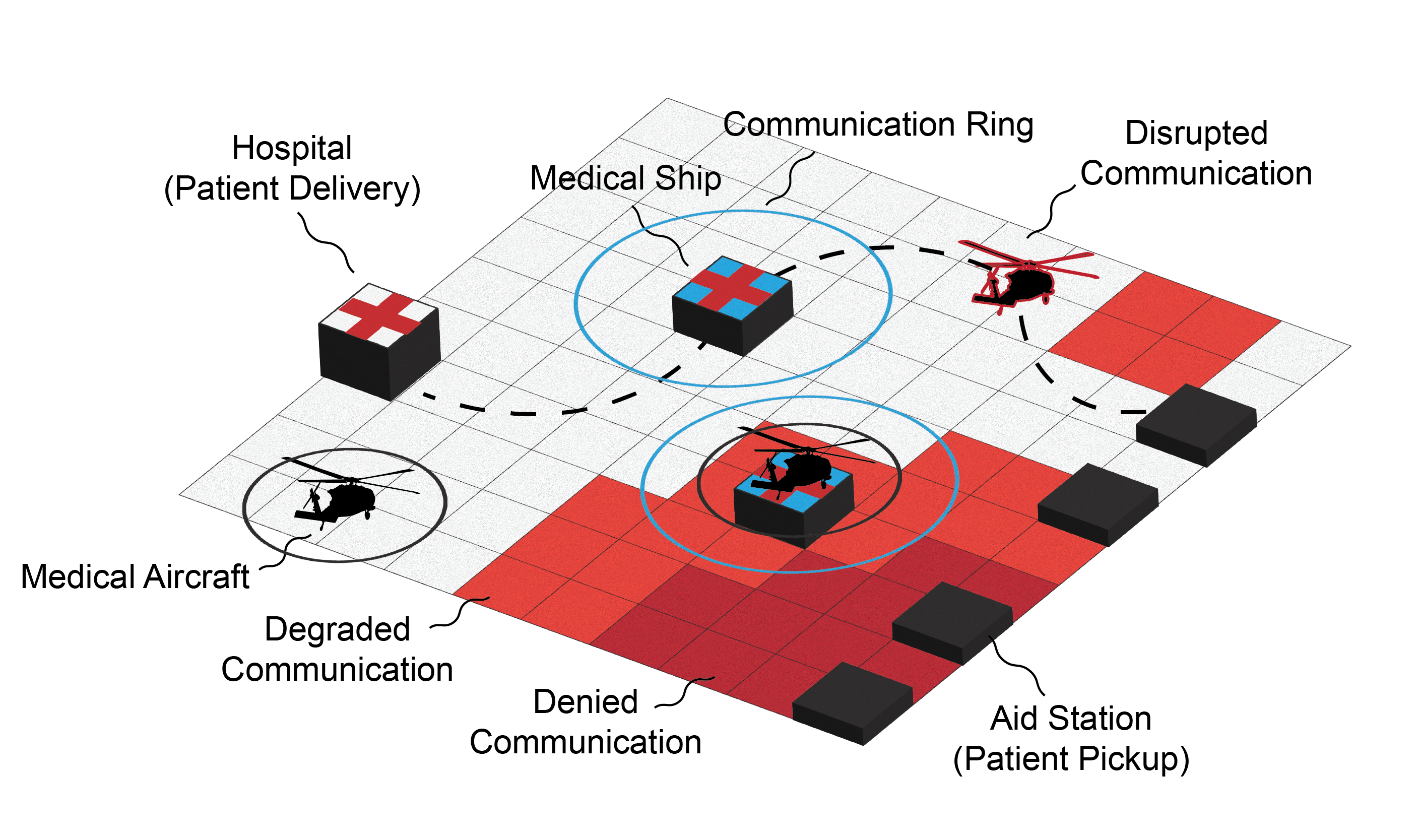}
  \caption{A semi-decentralized multiagent evacuation scenario with probabilistic restrictions on communication. Aircraft and watercraft must coordinate under communication constraints to move patients from aid stations to hospitals.}
  \label{fig:cover_photo}
\end{figure}

We begin by formally defining the \textit{semi-decentralization property}. The key insight is that \textbf{semi-Markov systems for resolving agent control may be co-opted in semi-decentralized systems for agent communication.}  Semi-decentralized systems may simultaneously contain a subset of agents that are necessarily decentralized and a subset of agents that are permissibly centralized. We introduce the semi-decentralized POMDP (SDec-POMDP), which unifies several existing multiagent models and explicit communication mechanisms. The SDec-POMDP reveals underlying mechanisms for memory propagation, selector functions, which we toggle here but are otherwise inherently present and constant in existing models. We then detail an exact optimal planning algorithm for solving SDec-POMDPs and apply to a set of new semi-decentralized benchmarks and a maritime medical evacuation application. 
\\~\\
\noindent
\textbf{Contributions}
\begin{itemize}
    \item We formulate the semi-decentralization property by extending semi-Markov control concepts to communication. Semi-decentralization is then applied to multiagent POMDPs, forming the SDec-POMDP. 
    \item We prove that the SDec-POMDP unifies the Dec-POMDP, the MPOMDP, $k$-steps delayed communication, and the Dec-POMDP-COM. 
    \item We introduce Recursive Small-Step Semi-Decentralized A* (RS-SDA*), an exact algorithm for solving semi-decentralized problems. We evaluate performance in semi-decentralized variants of four Dec-POMDP benchmarks, then apply to a complex medical evacuation scenario with joint tasks.
\end{itemize}

Ultimately, we describe and validate the mechanisms of a novel yet foundational property and model for multiagent decision making in probabilistic communication environments. 

\section{Related Work}


\noindent \textbf{Multiagent Models} Centralized planning can occur when agents freely communicate in reliable networks with no latency or when individual agents possess full system observability. This allows a single planner to select joint actions over the global state \cite{oliehoek2016concise}. A partially observable centralized multiagent system is represented by a multiagent partially observable Markov decision process (MPOMDP). MPOMDP planners benefit from conditioning joint actions on joint observations but scale poorly and are susceptible to communication failures. MPOMDPs therefore have limited application to many challenging real-world problems. By contrast, cooperative agents in decentralized systems can neither communicate explicitly nor observe the entire system state, and therefore must select their own actions in accordance with local observations \cite{kochenderfer2015decision}. The decentralized POMDP (Dec-POMDP) and MPOMDP share an underlying model but possess distinct policies and histories. Decentralized models have been used to support unmanned aerial vehicle formation flight \cite{azam2021uav}, maritime traffic management \cite{singh2019multiagent}, and wildfire surveillance \cite{julian2019distributed}. Decentralized systems permit implicit communication, in which agents transmit information by means of taken actions and received observations \cite{grover2010implicit}. Explicit communication, by contrast, endows agents with formalized communications actions.

The state-of-the-art exact optimal algorithm for solving Dec-POMDPs, recursive small-step multi-agent A* (RS-MAA*) \cite{koops2023recursive}, relies on a combination of incremental expansion, clustering, variable-depth recursive heuristics, and heuristic memoization. We extend RS-MAA* to semi-decentralized systems. 



\noindent \textbf{Communication Schemes} The literature features numerous explicitly modeled communication mechanisms and frameworks, owing to the sheer diversity of information structures in real-world problems. \citeauthor{goldman2003optimizing} formalize the decentralized POMDP with communication (Dec-POMDP-Com), which incorporates an alphabet of possible messages and a communication cost function \cite{goldman2003optimizing}. The Dec-POMDP-Com models costly communication and generally assumes noise-free instantaneous broadcasting. Delayed communication occurs when agents learn the local observation received by others after one or more time-steps \cite{oliehoek2016concise}. In one-step delayed communication, agents determine the latest joint action from the joint policy, but are unaware of the latest joint observation \cite{oliehoek2012tree}. In $k$-steps delayed communication, a sufficient statistic may be used in place of the past joint policy \cite{oliehoek2013sufficient}. Other multiagent models consider noisy \cite{tung2021effective} or intermittent \cite{zhang2024multiagent} communication channels. 

In the costly, delayed, noisy, and intermittent communication cases, the domain environment is orthogonal to the communication channel, and the actions taken by agents and resulting state do not affect their future ability to coordinate. Therefore, while communication directly influences control, control does not in turn directly influence communication. This distinguishes semi-decentralized infrastructure from existing communication schemes. 

\noindent \textbf{Communication in Reinforcement Learning} Our primary contribution is a foundational model that supports both planning and reinforcement learning in restricted communication environments. We further present an exact algorithm for semi-decentralized planning, in which the communication and model dynamics are known prior to execution \cite{moerland2022unifying}. Still, the growing body of multiagent deep reinforcement learning with communication (Comm-MADRL) research is replete with promising techniques for codifying communication and defining communication policy. \citeauthor{zhu2024survey} categorize Comm-MADRL approaches by communication policy, to include full communication, partial communication, individual control, and global control subcategories \cite{zhu2024survey}. We are chiefly concerned with the individual control literature \cite{jiang2018learning, singh2018learning, sheng2022learning}, and seek to learn optimal control policies in light of potential communication links between agents. In the sub-field of \textit{learning tasks with communication}, policies generated using learning algorithms simultaneously maximize environmental rewards and determine effective communication protocols for agents \cite{zhu2024survey}. This parallels recent efforts in planning to design joint communication and control strategies \cite{sudhakara2024optimal}, which we also do. 

\section{Preliminaries}

\subsection{Dec-POMDPs}

The \textit{decentralized partially observable Markov decision process} (Dec-POMDP) is a stochastic, decentralized multiagent model for sequential decision-making under partial observability characterized by tuple $\langle I, S, \bar{A}, \bar{\mathcal{O}}, T, O, R\rangle$, where:

\begin{itemize}
  \item $I$ is a finite set of $k$ agents,
  \item $S$ is a finite set of states,
  \item $\bar{A} = \times_i A_i$ is a finite set of joint actions,
  \item $\bar{\mathcal{O}} = \times_i \mathcal{O}_i$ is a finite set of joint observations, and
  \item $T: S \times \bar{A} \times S \rightarrow [0,1]$ is a state transition function where $T(s'\mid s, \bar{a})$ is the probability of being in state $s'$ given joint action $\bar{a}$ being performed in state $s$,
  \item $O: \bar{\mathcal{O}} \times S \times \bar{A} \rightarrow [0,1]$ is a joint observation function where $O(\bar{o}' \mid s', \bar{a})$ specifies the probability of attaining joint observation $\bar{o}'$ when joint action $\bar{a}$ results in state $s'$, and 
  \item $R: S \times \bar{A} \rightarrow \mathbb{R}$ is a reward function such that $R(s, \bar{a})$ is the immediate reward for performing joint action $\bar{a}$ in state $s$.
\end{itemize}

Agents in the \textit{decentralized partially observable Markov decision process} (Dec-POMDP) cannot explicitly share information. Each agent therefore selects actions in accordance with a local policy $\pi_i$ informed by action observation history $h_i$, where $\bar{h} \in (\prod_{i \in I}A_iO_i)^\star$. We are primarily concerned with deterministic policies, from which agent actions can be inferred using only observation histories $\bar{o}_h$. The collection of individual policies is the decentralized policy set $\bar{\pi}: \langle \pi_1, \pi_2,... \pi_n\rangle$. Assume an infinite horizon $h = \infty$ and time discount rate $\gamma$. The objective is to find a policy set $\bar{\pi}$ maximizing expected reward over states and observation histories: 

\begin{equation*}
     V^{\bar{\pi}}(s, \bar{o}_h) = \EX\ \left[\sum\limits_{t = 0}^{\infty} \gamma^t R(s^t, \bar{\pi}(\bar{o}_h^t)) \mid b^0, \bar{\pi} \right]
\end{equation*}

\begin{equation*}
    \begin{aligned}
    \bar{\pi}^{*} {\leftarrow} \argmax\limits_{\bar{a} \in \bar{A}} \left( R(s, \bar{a}) + \gamma \sum\limits_{s' \in S}  \sum\limits_{\bar{o}' \in \mathcal{O}} \text{Pr}(s', \bar{o}' {\mid} s, \bar{a}) V^{*}(s', \bar{o}_h') \right).
    \end{aligned}
\end{equation*}

\subsection{MPOMDPs}

The \textit{multiagent partially observable Markov decision process} (MPOMDP) is a stochastic, centralized multiagent model for sequential decision-making under partial observability also characterized by tuple $\langle I, S, \bar{A}, \bar{\mathcal{O}}, T, O, R \rangle$. MPOMDP agents also lack a complete picture of the underlying system state, but can share individual actions $a_i$ taken and observations received $o_i$. This permits a \textit{sufficient statistic} in the form of a probability distribution over states called a \textit{belief} $b \in \Delta^n$, where $n$ is the number of states and $\Delta^n$ is the $n-1$ simplex. We notate the multiagent belief using $b_\alpha$. The belief over successor states $s'$ is updated using the set of agent histories $\bar{h}$ of actions taken and observation received: 

\begin{equation*}
    b_\alpha'(s') = \eta  O(\bar{o}' \mid s', \bar{a}) \sum\limits_{s \in S} T(s'\mid s, \bar{a}) b_\alpha(s)
\end{equation*}

\noindent where $\eta$ is a normalizing constant equal to $\text{Pr}(\bar{o}' \mid b_\alpha, \bar{a})^{-1}$, and $\text{Pr}(\bar{o}' \mid b_\alpha, \bar{a}) = \sum\limits_{s' \in S} O(\bar{o}' \mid s', \bar{a}) \sum\limits_{s \in S}T(s'\mid s, \bar{a}) b_\alpha(s)$. Belief $b'_{_\alpha \bar{o} \bar{a}}$ is then generated by applying the belief update equation to all $s' \in S$. An initial belief $b^0$ is assumed to avoid considering the Bellman equation for infinite starting beliefs. From the agent perspective, reward must be calculated as a function of belief, such that $R(b_\alpha, \bar{a}) = \sum_{s \in S}R(s, \bar{a})b_\alpha(s)$. In an MPOMDP, we seek a single joint policy $\pi_\alpha: \Delta^n \rightarrow \bar{A}$ that maps beliefs to actions. The expected value of $b_\alpha$ over an infinite horizon may be written as:

\begin{equation*}
    V^{\pi_\alpha}(b_\alpha) = \EX\ \left[\sum\limits_{t = 0}^{\infty} \gamma^t R(b_\alpha^t, \pi_\alpha(b_\alpha^t)) \mid b_\alpha^0 = b_\alpha, \pi_\alpha \right].
\end{equation*}

\noindent We seek a $\pi^*_\alpha$ that maximizes expected reward over time, determined by:

\begin{equation*}
    \begin{aligned}
    \pi_\alpha^{*} \leftarrow \argmax\limits_{\bar{a} \in \bar{A}} \left( R(b_\alpha, \bar{a}) + \gamma \sum\limits_{\bar{o}' \in \bar{\mathcal{O}}} \text{Pr}(\bar{o}' \mid b_\alpha, \bar{a}) V^{*}(b'_{_\alpha \bar{o} \bar{a}}) \right)
    \end{aligned}.
\end{equation*}

\subsection{Semi-Markov Processes}
\noindent \textbf{Definition 1.} The \textit{semi-Markov property for control} admits a distribution over time for state transitions aligned with agent actions. 

\noindent \textbf{Definition 2.} \textit{Sojourn control time} $\tau$ is the assignment of general continuous random variable $\mathcal{T}$. 

\boxedeq{eq:first}{Q(\mathcal{T} \leq \tau, s' \mid s, a) }

\noindent The \textit{semi-Markov decision process} (SMDP) \cite{puterman1994markov} is a stochastic single agent model for sequential decision-making with sojourn system control. The SMDP is characterized by state transition distribution $Q$, which is the probability that the next state transition occurs at or before $\tau$ and results in successor state $s'$. It is mathematically convenient to define $Q$ as the product of $F(\tau \mid s, a) T(s' \mid s, a, \tau)$, or equivalently $F(\tau \mid s', a, s) T(s' \mid s, a)$, where $F$ is the cumulative distribution function of $\tau$. The distribution of $\tau$ is conditioned on the current state $s$ and action taken $a$ and is therefore Markov. The ``semi'' in semi-Markov reflects the arbitrary probability distribution followed by model transitions. When $\tau$ = 0, a new action is taken to coincide with a \textit{decision epoch}. The system natural process time is denoted as $\eta \in \mathbb{R}^+$, such that each decision epoch $t \in \mathbb{N}$ with corresponding sojourn time $\tau^t \in \mathbb{R}^+$ and state $s^t \in S$ occurs at $\eta^t$, where $\eta^t$ is decision epoch start time within the natural process. 


\section{Semi-Decentralization}
\noindent \textbf{Definition 3.} The \textit{semi-Markov property for communication}, or \textit{semi-decentralization}, admits a distribution over time for what information agents can store in memory.

\noindent \textbf{Definition 4.} \textit{Sojourn communication time} $\tau$ is general continuous random variable representing the time for an agent to return to an information-sharing state. We here overload $\tau$ for sojourn communication time, distinct from sojourn control time.

\boxedeq{eq:second}{Q(\bar{\mathcal{T}} \leq \bar{\tau}', s' \mid \bar{\tau}, s, \bar{a}, \bar{a}')}

\noindent As with SMDPs, we can define $Q$ as the product of $F(\bar{\tau}' \mid s', \bar{a}', \bar{\tau})$ and $T(s' \mid s, \bar{a}, \bar{\tau})$, where $\bar{\tau}'$ may be conditioned on the subsequent joint action set $\bar{a}'$. SMDPs have one agent with an implicit conditioned $\bar{\tau}=0$. However, SDec-POMDPs have multiple agents with varied $\bar{\tau}$. Thus it is more general with $\bar{\tau}'$ conditioned on $\bar{\tau}$. Semi-decentralized models assume an initial $\bar{\tau}^0$, which can be interpreted as the communicating state of each agent when $\eta = 0$. When $\tau$ = 0, information sharing can occur coinciding with a \textit{communication epoch}. We assume noise-free instantaneous broadcast communication resulting in a single communicating agent set as in a blackboard \cite{erman1980hearsay, craig1988blackboard}. Semi-decentralization may however incorporate multiple distinct communicating sets. Whereas as semi-Markov control systems toggle model transition dynamics using $\tau$, semi-decentralized systems toggle updating histories using $\bar{\tau}$. 


\section{The SDec-POMDP}
The \textit{semi-decentralized partially observable Markov decision process} (SDec-POMDP) is a stochastic, semi-decentralized multiagent model for sequential decision-making under partial observability characterized by tuple $\langle I, S, \bar{A}, \bar{\mathcal{O}}, F, T, O, R\rangle$. 

\textbf{Model} The \textbf{SDec-POMDP model} introduces selector functions $f$, $g$, and $h$ to propagate agent memories, actions, and observations, respectively, to $M_i$ and $M_c$ through selector sub-nodes conditioned on $\bar{\tau}$. Figure \ref{fig:SDecPOMDP} depicts an SDec-POMDP dynamic decision network where selector sub-nodes are contained within $Z_{i}^\text{sel} = \langle M_{ci}^\text{sel}, \bar{A}_{i}^\text{sel}, \bar{O}_{i}^\text{sel}\rangle$ and $Z_{c}^\text{sel} = \langle \bar{M}^\text{sel}, \bar{A}^\text{sel}, \bar{O}^\text{sel}\rangle$. The selector infrastructure defines information-sharing configurations and enables the model to simultaneously maintain sets of centralized and decentralized agents. The set compositions change with changes to $\bar{\tau}$ following each state transition. Define $M_c = (\prod_{i \in I}(\{\varnothing\} \cup A_iO_i))^\star$, and $M_i = (A_iO_i)^\star \times M_c$. The selector functions follow:

\begin{equation*}
    \forall i, f(m_{ci}^{\text{sel}}) =     
    \begin{cases}
      m_c & \text{if } \tau_i = 0\\
     \varnothing & \text{if } \tau_i > 0\\
    \end{cases}
\end{equation*}

\begin{equation*}
    f(m_{c}^{\text{sel}}) =     
    \begin{cases}
      \forall i, m_i & \text{if } \tau_i = 0\\
     \varnothing & \text{if } \tau_i > 0\\
    \end{cases}
\end{equation*}


\begin{equation*}
    \forall i, g(a_i^{\text{sel}}) = 
    \begin{cases}
     a_j \in \bar{a}, \forall j \mid \tau_j = 0 & \text{if } \tau_i = 0\\
     a_i & \text{if } \tau_i > 0\\
    \end{cases}
\end{equation*}

\begin{equation*}
    g(a_c^{\text{sel}}) = \{ a_i \in \bar{a}, \forall i \mid \tau_i = 0 \}
\end{equation*}


\begin{equation*}
    \forall i, h(o_i^{\text{sel}})    
    \begin{cases}
     o_j \in \bar{o}, \forall j \mid \tau_j = 0 & \text{if } \tau_i = 0\\
     o_i & \text{if } \tau_i > 0\\
    \end{cases}
\end{equation*}

\begin{equation*}
    h(o_c^{\text{sel}}) = \{ o_i \in \bar{o}, \forall i \mid \tau_i = 0 \}.
\end{equation*}


\begin{figure}[t!]
  \centering
  \includegraphics[width=1.0\linewidth]{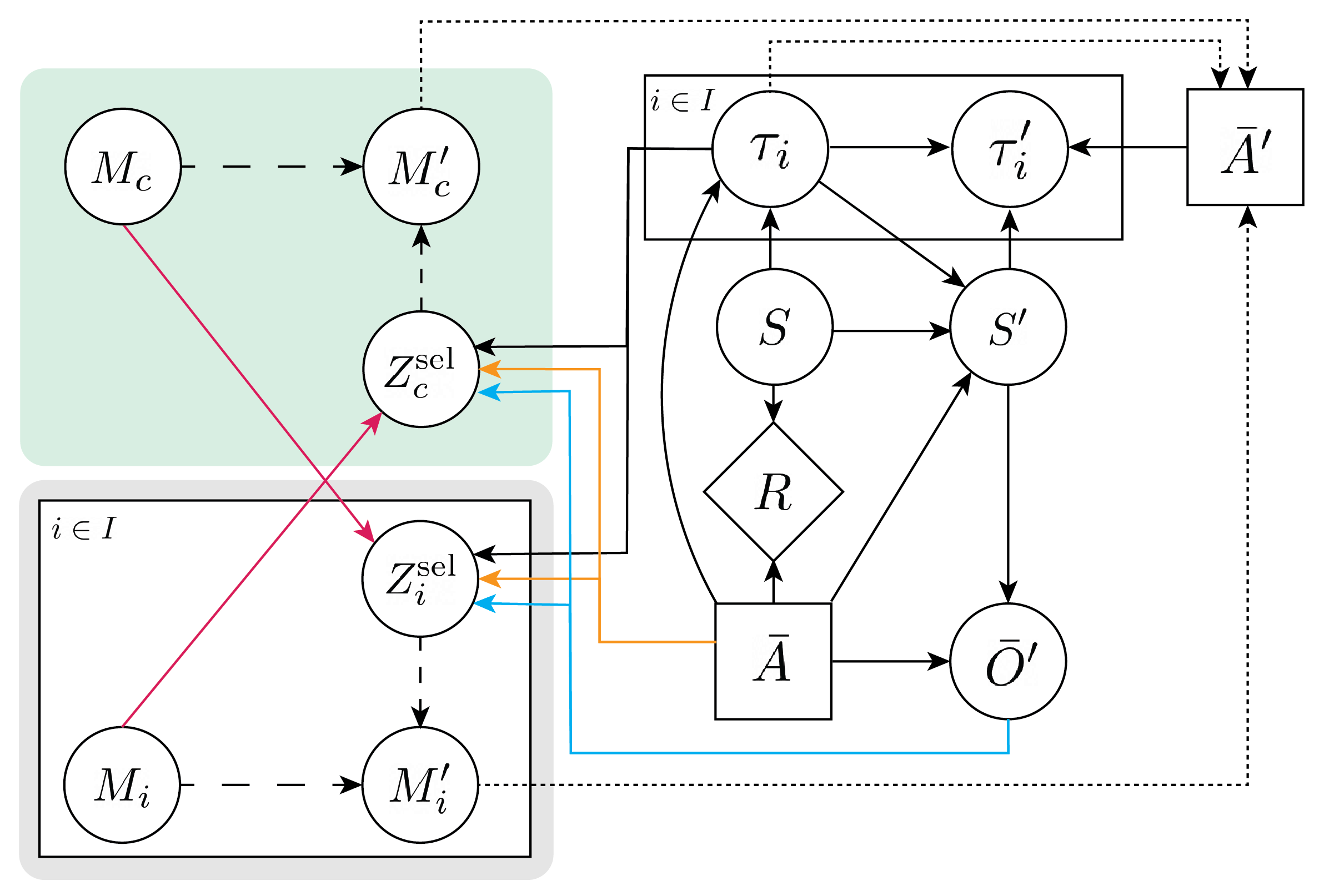}
  \caption{The SDec-POMDP dynamic decision network, with the policy infrastructure on the left and model on the right. The green backdrop contains the blackboard with memory $M_c$ generated from the histories of communicating agents. The gray backdrop with plate notation includes the individual agent memories $M_i$. $Z$ selector nodes are selectively toggled by $\bar{\tau}$ to facilitate memory propagation $\eta$, represented by dashed lines. Policy $\psi$ edges are represented by dotted lines. The SDec-POMDP framework is flexible and can be easily modified to capture the structural and informational characteristics of different problem domains.}
  \label{fig:SDecPOMDP}
\end{figure}

\noindent \textbf{Policy} For generality, assume an \textbf{SDec-POMDP policy} with deterministically stored $a_io_i, \forall i$ and deterministically selected $\bar{a}$, parameterized by both $\pi_i: M_i \rightarrow A$ and $\pi_c: M_c \rightarrow \bar{A}$. Memory propagation $\eta$ and policy $\psi$ probability functions follow: 

\begin{equation*}
    \psi(\bar{a}) = 
    \begin{cases} 1, \forall i
        \begin{cases}
          a_i = (\pi_c(m_c))_i & \text{if } \tau_i = 0\\
         a_i = \pi_i(m_i) & \text{if } \tau_i > 0\\
        \end{cases} \\
        0, \text{otherwise}
    \end{cases}
\end{equation*}

\begin{equation*}
        \eta_c(m_c') =     
        \begin{cases}1, \text{if } m_c' = m_c \bar{m}^\text{sel} \bar{a}^\text{sel} \bar{o}^\text{sel} \\
        0, \text{otherwise}
        \end{cases}
\end{equation*}

\begin{equation*}
        \forall i, \eta(m_i') =     
        \begin{cases}1, \text{if } m_i' = m_i m_{ci}^\text{sel} a_i^\text{sel} o_i^\text{sel} \\
        0, \text{otherwise}
        \end{cases}.
\end{equation*}


\noindent \textbf{Objective Function} The \textbf{SDec-POMDP objective} is to identify the combination of memory propagation and policy functions that will maximize expected reward. Consider the infinite horizon case: 

\begin{equation*}
    J(\psi, \eta_c, \bar{\eta}) = \EX\ \left[\sum\limits_{t = 0}^{\infty} \gamma^t R(s^t, \bar{a}^t) \mid b^0 \right]
\end{equation*}

\begin{equation*}
    J^\star = \argmax \limits_{\psi, \eta_c, \bar{\eta}} J(\psi, \eta_c, \bar{\eta}).
\end{equation*}

\noindent Rewrite in terms of policies $\bar{\pi}$ and $\pi_c$ and blackboard and agent memories $m_c$ and $\bar{m}$: 

\begin{align*}
    V^{\pi_c, \bar{\pi}}(m_c, \bar{m}) &  = \EX\ \left[\sum\limits_{t = 0}^{\infty} \gamma^t R(s^t, \bar{a}^t) \mid b^0, m_c, \bar{m}, \pi_c, \bar{\pi} \right] \\
    & = \sum\limits_{s \in S} V^{\pi_c, \bar{\pi}}(m_c, \bar{m}, s) b^0(s) 
\end{align*}

\begin{equation*}
    \begin{aligned}
        &  V^{\pi_c, \bar{\pi}}(m_c, \bar{m}, s) = \sum\limits_{\bar{a} \in \bar{A}} \psi(\bar{a} \mid m_c, \bar{m}) \left[ R(s, \bar{a}) + \gamma \sum\limits_{\bar{\tau} \in \bar{T}} F(d\bar{\tau} \mid s, \bar{a})  \right.\\ 
        & 
        \sum\limits_{s' \in S} T(s' \mid s, \bar{a}) 
        \sum\limits_{\bar{o}' \in \bar{\mathcal{O}}} O(\bar{o}' \mid s', \bar{a}) \\ 
        &  \sum\limits_{m_c' \in M_c} \eta_c (m_c' \mid m_c, f(\bar{m}), g(\bar{a}), h(\bar{o}')) \\ 
        & \left. \sum\limits_{\bar{m}' \in \bar{M}} \prod\limits_{i \in I} \eta(m_i' \mid m_i, f(m_c), g(\bar{a}), h(\bar{o}')) V^{\pi_c, \bar{\pi}}(m_c', \bar{m}', s') \right].
    \end{aligned}
\end{equation*}

\section{Theoretical Analysis}



\noindent \textbf{Definition 5.} Models $X_\theta$ and $X_\phi$ are equivalent if they reduce to one another via mapping function $f$, such that $X_\theta \leq X_\phi$ and $X_\phi$ $\leq$ $X_\theta$. 

\noindent \textbf{Definition 6.} Model-policy structures $XY_\theta$ and $XY_\phi$ are equivalent if they reduce to one another via mapping function $g$, such that $XY_\theta \leq XY_\phi$ and $XY_\phi$ $\leq$ $XY_\theta$. 

\noindent \textbf{Definition 7.} Model-policy-objective structures $Z_\theta$ and $Z_\phi$ are equivalent if $\forall \bar{h}$ $V_{XY_\theta}(\bar{h}, b^0) = V_{XY_\phi}(\bar{h}, b^0)$. 



\subsection{MPOMDP}
\begin{lemma}
SDec-POMDP and MPOMDP models are equivalent. 
\end{lemma}

\begin{proof}
Demonstrate that 1. $X_\text{MPOMDP}$ $\leq$ $X_\text{SDec-POMDP}$ and 2. $X_\text{SDec-POMDP} \leq X_\text{MPOMDP}$ \\

\noindent \textbf{1.} $X_\text{MPOMDP}$ $\leq$ $X_\text{SDec-POMDP}$

\noindent 
Let $I' = I$ , $S' = S$, $\bar{A}' = \bar{A}$, $R' = R$, and $\bar{\mathcal{O}}' = \bar{\mathcal{O}}$, where prime notation indicates the SDec-POMDP for purposes of relating models. 
The state transition and observation functions are defined to reproduce the MPOMDP dynamics independently of $\bar{\tau}$, such that $T'(s' \mid s, \bar{a}, \bar{\tau}) = T(s' \mid s, \bar{a})$ and $O'(\bar{o}' \mid s', \bar{a}, \bar{\tau})$ = $O(\bar{o}' \mid s', \bar{a})$. 
Assume a deterministic communication sojourn time function $F$, where $\tau'$ for each agent is fixed to one, resulting in complete centralization at subsequent decision epochs; $\forall i$, $F'(\tau_i' < 1  \mid s, a_i, \tau_i$) = 0 and $F'(\tau_i' = 1  \mid s, a_i, \tau_i$) = 1. Selector functions therefore return joint actions and observations at each time-step; $\forall i$, $g'(a_i^{\text{sel}}) = g'(a_{c}^{\text{sel}}) = \bar{a}$ and $\forall i$, $h'(o_i^{\text{sel}}) = h'(o_{c}^{\text{sel}}) = \bar{o}$. \\

\noindent \textbf{2.} $X_\text{SDec-POMDP}$ $\leq$ $X_\text{MPOMDP}$

\noindent 
Let $R' = R$ and $\bar{A}' = \bar{A}$. The agent set is extended to include an additional blackboard agent with an independent memory, so that $I' = I \cup I_\text{c}$. We assume without loss of generality that $|I_\text{c}|=1$, and that the communication sojourn time for this agent satisfies $\tau_i = 0$ $\forall i \in I_\text{c}$. The state space is augmented to include $\bar{\tau}$, such that $S' = S \times (\mathbb{R}^+)^n$. Similarly, the joint observation space is expanded to include the sequence of all action-observation pairs,
\begin{equation*}
\bar{\mathcal{O}}' = \bar{\mathcal{O}} \times \displaystyle \prod_{i \in I}(A_iO_i)^\star. 
\end{equation*}
The transition and observation functions adopt the factored state space: $T'(\langle s', \bar{\tau}'\rangle \mid \langle s, \bar{\tau} \rangle, \bar{a}, \bar{a}') = F(\bar{\tau}' \mid s', \bar{a}', \bar{\tau}) T(s' \mid s, \bar{a}, \bar{\tau})$ and $O'(\bar{o}' \mid \langle s', \bar{\tau}' \rangle,  \langle s, \bar{\tau} \rangle, \bar{a}) = O(\bar{o}' \mid s', s, \bar{a}, \bar{\tau}', \bar{\tau})$ at its most general. 

\end{proof}

\begin{lemma}
SDec-POMDP and MPOMDP model-policy structures are equivalent.  
\end{lemma}

\begin{proof} Demonstrate that 1. $XY_\text{MPOMDP}$ $\leq$ $XY_\text{SDec-POMDP}$ and 2. $XY_\text{SDec-POMDP} \leq XY_\text{MPOMDP}$ \\

\noindent \textbf{1.} X$Y_\text{MPOMDP}$ $\leq$ $XY_\text{SDec-POMDP}$

\noindent
Let each memory selector function return the joint memory, such that $\forall i$, $f'(m_{ci}^\text{sel}) = f'(m_{c}^{\text{sel}}) = \bar{m}$. The update rule for the joint memory is deterministic and concatenates the prior shared memory with the complete set of agent actions and observations from the current time-step, as represented by: 
\begin{equation*}
    \eta'(\bar{m}') = \eta_c'(m_c') = \begin{cases}1, \text{if } \bar{m}' = \bar{m} \bar{a}^\text{sel} \bar{o}^\text{sel} \\
        0, \text{otherwise.}
        \end{cases}
\end{equation*}
Agent actions are selected using a policy over the joint memory: 
\begin{equation*}
    \psi'(\bar{a}) = 
    \begin{cases} 1, \bar{a} = \pi(\bar{m}) \\
        0, \text{otherwise.}
    \end{cases}
\end{equation*}

\noindent \textbf{2.} $XY_\text{SDec-POMDP} \leq XY_\text{MPOMDP}$ \\
\noindent 
Construct policy $\pi^M(m_i') = \langle a_1,...,a_n\rangle$, simulating where the action for agent $i$ is determined to be:
\begin{equation*}
a_i = \begin{cases} (\pi_c(\bar{m}_c))_i & \text{if } \tau_i = 0 \\ \pi_i(m_i) & \text{if } \tau_i > 0 \end{cases}
\end{equation*}
as if each agent draws from the blackboard's policy when its communication sojourn time is zero and otherwise following a local policy. The memory update rule is defined to extend the current memory with the observed joint outcome:
\begin{equation*}
\eta'(\bar{m}') = \begin{cases}1, \text{if } \bar{m}' = \bar{m} \bar{o} \\
        0, \text{otherwise.}
        \end{cases}
\end{equation*}
Joint action selection is consistent with the constructed joint policy:
\begin{equation*}
\psi'(\bar{a}) = 
    \begin{cases} 1, \forall i, 
        a_i = \pi^M(m_i') \\
        0, \text{otherwise.}
    \end{cases}
\end{equation*}

\end{proof} 

\begin{lemma} 
SDec-POMDP and MPOMDP model-policy-objective structures are equivalent. 
\end{lemma} 
\begin{proof} Show $\forall \bar{h}$ $V_{XY_\text{MPOMDP}}(\bar{h}, b_0) = V_{XY_\text{SDec-POMDP}}(\bar{h}, b_0)$ \\

Show that the semi-decentralized value function reduces to the standard value function under the original joint policy. Begin with,
\\
\noindent $ \begin{aligned}
    &  V^{\pi_c, \bar{\pi}}(m_c, \bar{m}, s) = \underbrace{\sum\limits_{\bar{a} \in \bar{A}} \psi(\bar{a} \mid m_c, \bar{m})}_{\textcolor{red}{1, m_c = \bar{h}}} \Biggl[ R(s, \bar{a}) + \gamma \underbrace{\sum\limits_{\bar{\tau} \in \bar{T}} F(d\bar{\tau} \mid s, \bar{a})}_{\textcolor{red}{1, \text{ all RV } \ind \text{ of } d\tau}} \\
    & \sum\limits_{s' \in S} T(s' \mid s, \bar{a}) 
    \sum\limits_{\bar{o}' \in \bar{\mathcal{O}}} O(\bar{o}' \mid s', \bar{a}) \\
    & \underbrace{\sum\limits_{m_c' \in M_c} \eta_c (m_c' \mid m_c, f(\bar{m}), g(\bar{a}), h(\bar{o}'))}_{\textcolor{red}{1, m_c' = m_c \bar{a}^\text{sel} \bar{o}^\text{sel}}} \\
    & \underbrace{\sum\limits_{\bar{m}' \in \bar{M}} \prod\limits_{i \in I} \eta(m_i' \mid m_i, f(m_c), g(\bar{a}), h(\bar{o}'))}_{\textcolor{red}{1, \forall i, m_i' = m_i a_i^\text{sel} o_i^\text{sel}}}
    V^{\pi_c, \bar{\pi}}(m_c', \bar{m}', s') \Biggr]. 
    \end{aligned} $ 

\noindent By construction, each under-braced term evaluates deterministically: the blackboard’s memory update enforces $m_c' = m_c \bar{a}^\text{sel} \bar{o}^\text{sel}$, each agent's local memory update yields $m_i' = m_i a_i^\text{sel} o_i^\text{sel}$, and the distribution over sojourn times collapse to one. Moreover, since $\psi(\bar{a} \mid m_c, \bar{m}, s) = 1$ whenever $\bar{a} = \pi(\bar{h})$, we obtain $V^{\pi_c, \bar{\pi}}(m_c, \bar{m}, s) = V^{\pi}(s, \bar{h})$, which results in standard value recursion: 
\begin{equation*}
R(s, \pi(\bar{h})) + \gamma \sum\limits_{s' \in S}  \sum\limits_{\bar{o}' \in \mathcal{O}} \text{Pr}(s', \bar{o}' \mid s, \pi(\bar{h})) V^{\pi}(s', \bar{h}').
\end{equation*}

\noindent $ \begin{aligned}    
    & \text{Further observe that } \bar{m}_c = \bar{m} \text{ and }  \\
    & \sum\limits_{s' \in S} T(s' \mid s, \bar{a}) \sum\limits_{\bar{o}' \in \bar{\mathcal{O}}} O(\bar{o}' \mid s', \bar{a}) = \sum\limits_{s' \in S}  \sum\limits_{\bar{o}' \in \mathcal{O}} \text{Pr}(s', \bar{o}' \mid s, \pi(\bar{h})).
\end{aligned} $


\end{proof} 
\begin{proposition}
The SDec-POMDP and MPOMDP are equivalent. 
\end{proposition}

\subsection{Dec-POMDP}

\begin{lemma}
SDec-POMDP and Dec-POMDP models are equivalent.
\end{lemma}

\begin{proof}

1. Demonstrate that $X_\text{Dec-POMDP}$ $\leq$ $X_\text{SDec-POMDP}$ and 2. $X_\text{SDec-POMDP} \leq X_\text{Dec-POMDP}$ \\

\noindent \textbf{1.} $X_\text{Dec-POMDP}$ $\leq$ $X_\text{SDec-POMDP}$
\\
Again set $I' = I$ , $S' = S$, $\bar{A}' = \bar{A}$, $R' = R$, and $\bar{\mathcal{O}}' = \bar{\mathcal{O}}$. Let $T'(s' \mid s, \bar{a}, \bar{\tau}) = T(s' \mid s, \bar{a})$ and $O'(\bar{o}' \mid s', \bar{a}, \bar{\tau}')$ = $O(\bar{o}' \mid s', \bar{a})$. Action and observation selection are specified so that for every agent $i$, $g'(a_i^{\text{sel}}) = a_i$ and $h'(o_i^{\text{sel}}) = o_i$. Assign to the blackboard memory the null set, such that $g'(a_{c}^{\text{sel}})$ = $h'(o_{c}^{\text{sel}})$ = $\varnothing$. By the construction of $g'$ and $h'$, $\tau$ has no impact and $F$ can be disregarded. 

\noindent \textbf{2.} $X_\text{SDec-POMDP}$ $\leq$ $X_\text{Dec-POMDP}$

\noindent Reference Lemma 1 proof 2, as $X_\text{Dec-POMDP} = X_\text{MPOMDP}$.

\end{proof}

\begin{lemma}
SDec-POMDP and Dec-POMDP model-policy structures are equivalent.
\end{lemma}  

\begin{proof}
Demonstrate that 1. $XY_\text{Dec-POMDP}$ $\leq$ $XY_\text{SDec-POMDP}$ and 2. $XY_\text{SDec-POMDP} \leq XY_\text{Dec-POMDP}$ \\

\noindent \textbf{1.} $XY_\text{Dec-POMDP}$ $\leq$ $XY_\text{SDec-POMDP}$

\noindent
Let each memory selector function return the null set, such that $\forall i$, $f'(m_{ci}^\text{sel}) = f'(m_{c}^{\text{sel}}) = \varnothing$. The update rule for each agent's memory is deterministic and concatenates the prior memory with the set of individual agent actions and observations from the current time-step, as represented by: 
\begin{equation*}
    \eta'(m_i') = \begin{cases}1, \text{if } m_i' = m_i a_i^\text{sel} o_i^\text{sel} \\
        0, \text{otherwise}
        \end{cases}
\end{equation*}
Agent actions are selected using a policy over the agent's memory:
\begin{equation*}
    \psi'(\bar{a}) = 
    \begin{cases} 1, \forall i, 
        a_i = \pi_i(m_i) \\
        0, \text{otherwise}
    \end{cases}
\end{equation*}
By the construction of $\psi'$, $\eta_c'$ can be disregarded. \\

\noindent \textbf{2.} $XY_\text{SDec-POMDP} \leq XY_\text{Dec-POMDP}$ \\
Construct simulated policy $(\pi^D)_i(m_i') = a_i$, where:
\begin{equation*}
    a_i = \begin{cases} (\pi_c(m_c))_i & \text{if } \tau_i = 0 \\
        \pi_i(m_i) & \text{if } \tau_i > 0
        \end{cases}
\end{equation*}
as if each agent draws from the blackboard’s policy when its communication sojourn time is zero and otherwise following a local policy. Each agent's memory update rule appends the taken action and observation to the current memory:
\begin{equation*}
\eta'(m_i') = \begin{cases} 1, \text{if } m_i' = m_i a_i o_i \\
0, \text{otherwise} \end{cases}
\end{equation*}

Agent action selection is consistent with the constructed policy:
\begin{equation*}
\psi'(\bar{a}) = 
    \begin{cases} 1, \forall i, 
        a_i = (\pi^D)_i(m_i') \\
        0, \text{otherwise}
    \end{cases}
\end{equation*}
\end{proof}

\begin{lemma}
    SDec-POMDP and Dec-POMDP model-policy-objective structures are equivalent. 
\end{lemma}

\begin{proof}
Show $\forall \bar{h}$ $V_{XY_\text{Dec-POMDP}}(\bar{h}, b_0) = V_{XY_\text{SDec-POMDP}}(\bar{h}, b_0)$
\\
\noindent
The semi-decentralized value function reduces to the decentralized value function under the original policy set, beginning with: \\
\noindent 
$ \begin{aligned}
&  V^{\pi_c, \bar{\pi}}(\underbrace{\cancel{m_c}}_{\textcolor{red}{M_c = \varnothing}}, \bar{m}, s) = \underbrace{\sum\limits_{\bar{a} \in \bar{A}} \psi(\bar{a} \mid \underbrace{\cancel{m_c}}_{\textcolor{red}{M_c = \varnothing}}, \bar{m})}_{\textcolor{red}{1}} \Biggl[ R(s, \bar{a}) + \\
& \gamma \underbrace{\sum\limits_{\bar{\tau} \in \bar{T}} F(d\bar{\tau} \mid s, \bar{a})}_{\textcolor{red}{1, \text{ all RV } \ind \text{ of } d\tau}} \sum\limits_{s' \in S} T(s' \mid s, \bar{a}) \sum\limits_{\bar{o}' \in \bar{\mathcal{O}}} O(\bar{o}' \mid s', \bar{a}) \\ 
& \underbrace{\sum\limits_{\bar{m}' \in \bar{M}} \prod\limits_{i \in I} \eta(m_i' \mid m_i, \underbrace{\cancel{f(m_c)}}_{\textcolor{red}{M_c = \varnothing }}, g(\bar{a}), g(\bar{a}), h(\bar{o}'))}_{\textcolor{red}{1, \forall i, m_i' = m_i a_i^\text{sel} o_i^\text{sel}}} \\
& \underbrace{\cancel{\sum\limits_{m_c' \in M_c} \eta_c (m_c' \mid m_c, f(\bar{m}), g(\bar{a}), h(\bar{o}'))}}_{\textcolor{red}{M_c = \varnothing}} V^{\pi_c, \bar{\pi}}(\underbrace{\cancel{m_c'}}_{\textcolor{red}{M_c = \varnothing}}, \bar{m}', s') \Biggr]. \end{aligned} $

\noindent By construction, each under-braced term evaluates deterministically: the blackboard’s memory remains $m_c' = \varnothing$, each agent's local memory update yields $m_i' = m_i a_i^\text{sel} o_i^\text{sel}$, and the distribution over sojourn times collapses to one. As before, since $\psi(\bar{a} \mid m_c, \bar{m}, s) = 1$ whenever $\bar{a} = \bar{\pi}(\bar{h})$, we obtain $V^{\pi_c, \bar{\pi}}(m_c, \bar{m}, s) =V^{\bar{\pi}}(s, \bar{h})$, which results in standard value recursion: 
\begin{equation*}
R(s, \bar{\pi}(\bar{h})) + \gamma \sum\limits_{s' \in S}  \sum\limits_{\bar{o}' \in \mathcal{O}} \text{Pr}(s', \bar{o}' \mid s, \bar{\pi}(\bar{h})) V^{\bar{\pi}}(s', \bar{h}') \\
\end{equation*}
Again observe that,
\begin{equation*}
\sum\limits_{s' \in S} T(s' \mid s, \bar{a}) \sum\limits_{\bar{o}' \in \bar{\mathcal{O}}} O(\bar{o}' \mid s', \bar{a}) = \sum\limits_{s' \in S}  \sum\limits_{\bar{o}' \in \mathcal{O}} \text{Pr}(s', \bar{o}' \mid s, \bar{\pi}(\bar{h})).
\end{equation*} 

\end{proof}

\begin{proposition}
    The SDec-POMDP and Dec-POMDP are equivalent
\end{proposition} 

    \noindent \textbf{Corollary 1.} The SDec-POMDP is the same complexity class as a Dec-POMDP (NEXP-complete)

\begin{proposition}
    The SDec-POMDP and $k$-steps-delayed communication are equivalent
\end{proposition} 

\begin{proposition}
    The SDec-POMDP and Dec-POMDP-Com are equivalent
\end{proposition} 

\textbf{We provide proofs for Propositions 3 and 4 in the technical appendix.}

\section{Recursive Small-Step Semi-Decentralized A*}
\textbf{Recursive small-step semi-decentralized A*} (RS-SDA*) is an exact planning algorithm for optimizing expected reward in SDec-POMDP problems. RS-SDA* modifies RS-MAA* \cite{koops2023recursive} by maintaining a stage-specific partition of decentralized and centralized joint observation histories per probabilistic communication dynamics. Like RS-MAA*, RS-SDA* relies on incremental expansion of a small-step search tree, clustering, recursive heuristics, memoization, and last stage modifications. RS-SDA* generates a fully-specified policy set $\bar{\pi} \in \Pi$ using offline planning. A fully-specified policy contains both fully-specified local policies and, if appropriate for the problem, a blackboard policy $\bar{\pi} = \langle \pi_1,...\pi_n, \pi_c \rangle$ where $\pi_i:\mathcal{O}^h \rightarrow A_i$ and $\pi_c:\bar{\mathcal{O}}^h \rightarrow \bar{A}$. Similarly, $ \bar{\varphi} = \langle \varphi_1,...\varphi_n, \varphi_c\rangle$, $\varphi_i:\mathcal{O}^{\leq h-1} \rightarrow A_i$, and $\bar{\varphi}:\bar{\mathcal{O}}^{\leq h-1} \rightarrow \bar{A}$. RS-SDA* is outlined in Algorithm \ref{alg:rssda}. 

\noindent \textbf{Small-Step Search Tree} 
We adopt the small-step approach first introduced by \citeauthor{cazenave2010partial} for A* \cite{cazenave2010partial} and used by \citeauthor{koops2023recursive} in RS-MAA* \cite{koops2023recursive} to limit search tree outdegree. Small-step search can be used with centralized and decentralized components of each policy node, depicted in Figure \ref{fig:RSSDA_offline}. As shown in Table \ref{table:nodes}, small-step search provides RS-SDA* with mixed component policies, pre-clustering, a lower bound (complete decentralization) and an upper bound (complete centralization) on the number of considered nodes per stage $t$. When $F(\tau \mid s, \bar{a}) = F (\tau \mid \bar{a})$ and the policy is deterministic, we can consolidate the centralized and decentralized components of policies and significantly reduce the search tree size. This results in an RS-SDA* lower bound that is equivalent to RS-MAA* and an RS-SDA* upper bound that has $|\mathcal{O}_\star|^{tn}$ levels per stage and considers $|A_\star|^n$ joint actions per level, where $|\mathcal{O}_\star|$ is the size of the largest observation set and $|A_\star|$ is the size of the largest action set. 

\begin{table}[htbp]
\centering
\resizebox{\columnwidth}{!}{%
\begin{tabular}{l|ll|ll|ll}
\toprule
\multicolumn{1}{c|}{} & \multicolumn{2}{c|}{\textbf{RS-MAA*}} & \multicolumn{2}{c|}{\textbf{RS-SDA*}} & \multicolumn{2}{c}{\textbf{Classical MAA*}}  \\

& \multicolumn{2}{c|}{\textit{(RS-SDA* lower bound)}} & \multicolumn{2}{c|}{\textit{upper bound}} & & \\
& levels & nodes/stage & levels & nodes/stage & levels & nodes/stage \\
$t$ & $n|O_\star|^t$ & $n|\mathcal{O}_\star|^t|A_\star|$ & $|\mathcal{O}_\star|^{tn}$ & $|\mathcal{O}_\star|^{tn}|A_\star|^n$ & 1 & $|A_\star|^{n|\mathcal{O}_\star|^t}$ \\
\midrule
0 & 2 & 6 & 1 & 9 & 1 & 9\\
1 & 4 & 12 & 4 & 36 & 1 & 81 \\
2 & 8 & 24 & 16 & 144 & 1 & 6,561 \\
3 & 16 & 48 & 64 & 576 & 1 & >1E6 \\
4 & 32 & 96 & 256 & 2,304 & 1 & >1E6\\
5 & 64 & 192 & 1,024 & 9,216 & 1 & >1E6\\
6 & 128 & 384 & 4,096 & 36,864 & 1 & >1E6 \\
7 & 256 & 768 & 16,384 & 147,456 & 1 & >1E6 \\
8 & 512 & 1536 & 65,536 & 589,824 & 1 & >1E6 \\
\bottomrule
\end{tabular}}
\caption{\textsc{SDec-Tiger} levels and nodes by stage for RS-MAA*, RS-SDA*, and MAA*. $|A_\star| = 3$, $|O_\star| = 2$, and $n=2$.}
\label{table:nodes}
\end{table}

\begin{figure}[ht]
  \centering
  \includegraphics[width=0.9\linewidth]{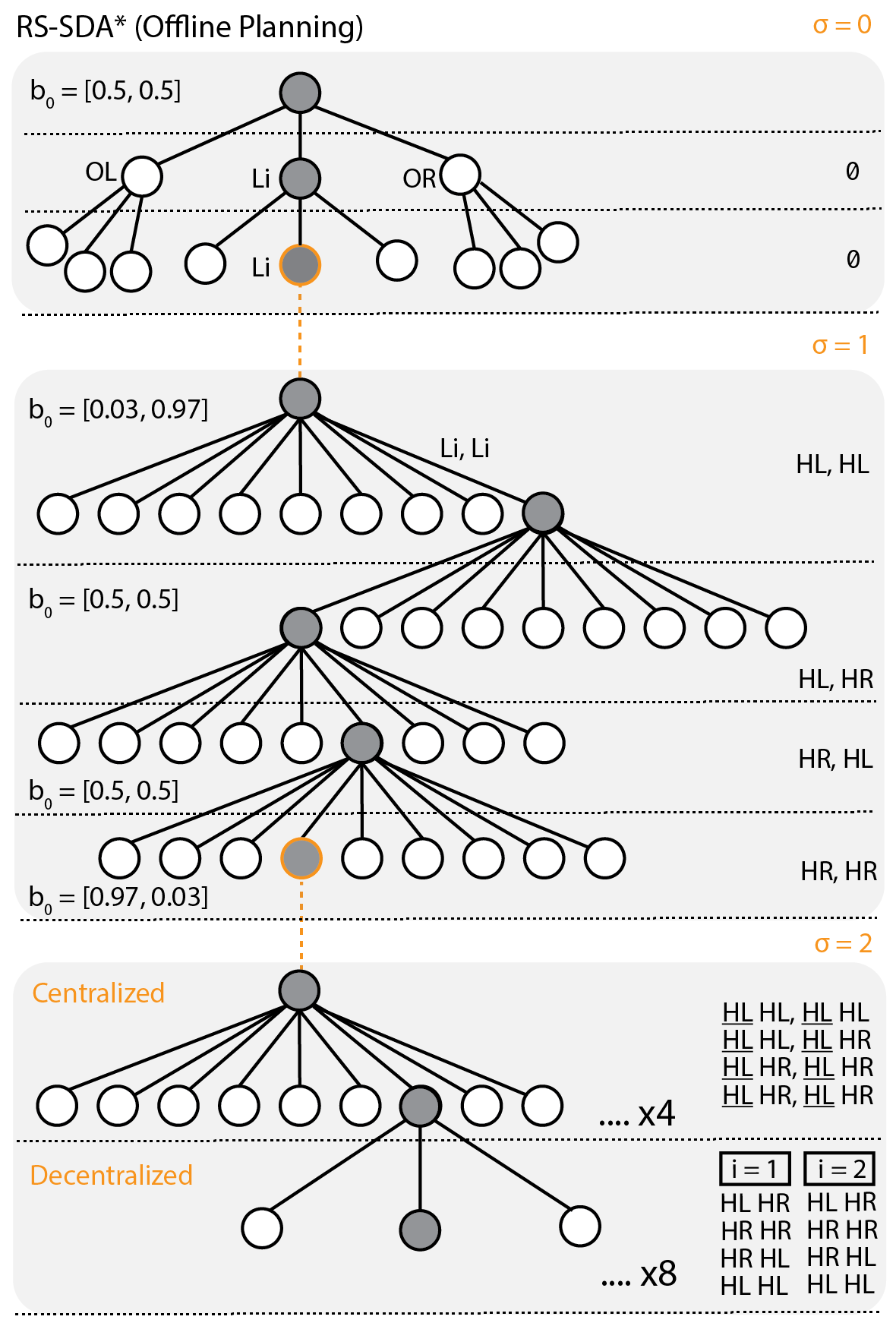}
  \caption{Illustrating RS-SDA* applied to \textsc{SDec-Tiger} using mixed component policies through stage $\sigma = 2$.}
  \label{fig:RSSDA_offline}
\end{figure}

\noindent \textbf{Dynamic Programming} We apply backward induction over beliefs to rapidly determine the value of centralized policy components. This bypasses expensive recursive heuristic calculations for large portions of the search. For each remaining horizon $r$ and belief $b$, we compute $ V_r(b)$ and $Q_r(b, \bar{a})$ and memoize both $V$ and $Q$ under keys $(r,b)$ and $(r,b,\bar{a})$ to enable extensive reuse during A* expansion. Similarly, observation likelihoods $P(\bar{o}\mid b, \bar{a})$ and posteriors $(b'_{\bar{o}, \bar{a}})$ are obtained via a belief update over the model’s transition and observation tensors, then cached for subsequent calls.

\noindent \textbf{Ordering Observation Histories} Any $\bar{\varphi}$ may contain both decentralized and centralized mappings conditioned on the underlying state and actions taken. We therefore explore joint observation histories (JOH) and local observation histories (LOH) in a predetermined sequence: stage, JOH then LOH, by agent (for LOHs), then lexicographically. Observe that, for length $t$, all $\bar{o}^{0:t} \preceq o^{0:t}$ and all $o^{0:t} \preceq \bar{o}^{0:t'}$. Additionally, $\bar{o}^{0:t} \preceq_\text{lex} \bar{o}^{0:t}$, and $o_i^{0:t} \preceq o_j^{0:t}$ if $(i < j)$ or $(i = j \land o_i^{0:t} \preceq_\text{lex} o_j^{0:t})$. 

\noindent \textbf{Clustering} We implement lossless incremental clustering in decentralized policy components based on a probabilistic equivalence criterion \cite{oliehoek2009lossless}. We similarly cluster centralized policy components based on the resulting joint belief, or $\forall_s$, $ \text{P}(s \mid \bar{o}^{0:t}_1) = \text{P}(s \mid \bar{o}^{0:t}_2)$.

\noindent \textbf{Admissible Heuristic} As with multiagent A* (MAA*) \cite{szer2005maa}, an \textit{admissible heuristic} $Q$ guides a path through a search tree with partial policies as nodes $\bar{\varphi}$. A heuristic is admissible if it equals or over-approximates the true value of the policy node. An \textit{open list} is maintained with nodes under consideration. The node in the open list with the highest heuristic value is expanded and replaced in the open list by its children. The tree search terminates once a fully-specified policy with the highest heuristic value is identified. 
For each parent node and candidate action, we split each posterior belief using $\mathcal{S}_{\text{sync}}$ and $\mathcal{A}_{\text{sync}}$ (more generally expressed using $F(s,a)$) and take a probability‑weighted sum of the exact centralized value on the communication-dependent posterior and the exact decentralized value on its complement. Because every constituent (centralized, decentralized, and their mixture conditioned on communication) is an exact optimum of a relaxation of the remaining subproblem, the resulting heuristic \textbf{never underestimates} the achievable return and is therefore admissible.

\begin{algorithm}[t]
\footnotesize
\caption{Recursive Small-Step Semi-Decentralized A*}
\label{alg:rssda}
\DontPrintSemicolon
\SetKwInOut{Input}{Input}
\SetKwInOut{Output}{Output}
\Input{$\Phi \triangleq (\mathcal{S}_{\text{sync}}, \mathcal{A}_{\text{sync}})$; $h$ horizon; $b$ initial belief; $\varphi$ initial policy; $d$ heuristic depth; $M$ iterations; $u$ upper bound}
\Output{optimal policy $\varphi^*$, optimal value $v^*$}
\BlankLine
\SetKwProg{Fn}{function}{:}{}
\Fn{\textsc{RS-SDA*}$(h, b, \varphi, d, M, u, \Phi)$}{
  $Q \leftarrow \textsc{PriorityQueue}(\uparrow)$\; $Q.\mathrm{push}(\min(\varphi.\mathrm{heuristics}), \varphi)$\; $i \leftarrow 0$ \;
  \While{$Q$ is not empty}{
    $(v,\varphi)\leftarrow Q.\mathrm{pop}()$\;
    \lIf{$v < \infty$}{$i \leftarrow i+1$}
    \lIf{$i \ge M$ or $v \le u$}{\Return $(\min(v,u), \mathrm{None})$}
    $\sigma \leftarrow$ current stage of $\varphi$; \quad $j \leftarrow$ current agent index of $\varphi$ \;
    
    \tcp{Phase 1: Transition to Subsequent Stage}
    \If{stage $\sigma$ is complete ($j = |A|$)}{
      \lIf{$\sigma = h$}{\Return $(v,\varphi)$}
      $(D_{\mathrm{dec}},D_{\mathrm{cen}},P_{\mathrm{dec}},P_{\mathrm{cen}},\rho) \leftarrow \textsc{TerminalProbs}(\varphi, \Phi)$ \; 
      \lIf{$\rho > 0$}{$\varphi \leftarrow \textsc{ClusterPolicyDec}(\varphi, D_{\mathrm{dec}}, P_{\mathrm{dec}})$}
      \lIf{$\rho < 1$}{$\varphi \leftarrow \textsc{ClusterPolicyCen}(\varphi, D_{\mathrm{cen}}, P_{\mathrm{cen}})$}
      $w\leftarrow \textsc{EvaluatePolicy}(\varphi, \rho)$ \; 
      $\varphi.\mathrm{heuristics}.\mathrm{append}(w)$ \;
      $\varphi.\mathrm{depth}\leftarrow \min(\sigma, d)$ \;
      $\varphi.\mathrm{extend\_empty\_stage}()$; \quad $\sigma \leftarrow \sigma+1$; \quad $j \leftarrow 0$ \;
    }

    \tcp{Phase 2: Policy Expansion}
    \uIf{$\sigma = h$ \emph{(Final Stage Optimization)}}{
        \uIf{$\rho < 1$ and centralized unfilled}{
            determine $\bar{a}^*$ for centralized component \;
            $\varphi' \leftarrow \varphi$ with $\bar{a}^*$ set; \quad $w \leftarrow\ \textsc{EvaluatePolicy}(\varphi', \rho)$ \;
            $Q.\mathrm{push}(w, \varphi')$
        }
        \uElse{
            determine $a_j^*$ for decentralized component \;
            $\varphi' \leftarrow \varphi$ with $a_j^*$ set; \quad $w \leftarrow\ \textsc{EvaluatePolicy}(\varphi', \rho)$ \;
            \lIf{$w = v$}{\Return $(w, \varphi' )$}
            $Q.\mathrm{push}(w, \varphi')$
        }
    }
    \uElse{
        \uIf{$\rho < 1$ and centralized unfilled}{
          \ForEach{joint action $\bar{a}\in \bar{A}$}{
            $\varphi' \leftarrow \varphi$ with $\bar{a}$ set \;
            $w \leftarrow\ \textsc{ExactCentralQ}(\varphi', \bar{a})$ \;
            $Q.\mathrm{push}(w, \varphi')$
          }
        }
        \uElseIf{agent $j$ unfilled}{
          \ForEach{$a_j\in A_j$}{
            $\varphi' \leftarrow \varphi$ with $a_j$ set \;
            $w\!\leftarrow\!\textsc{EvaluatePolicy}(\varphi')$ \; 
            $Q.\mathrm{push}(w, \varphi')$
          }
        }
    }
  }
}
\end{algorithm}

\section{Experiments}
\noindent \textbf{Semi-Decentralized Benchmarks} We evaluate RS-SDA* in semi-decentralized versions of four common Dec-POMDP benchmarks: \textsc{Dec-Tiger} \cite{nair2003taming}, \textsc{FireFighting} \cite{oliehoek2008optimal}, \textsc{BoxPushing} \cite{seuken2007improved}, and \textsc{Mars} \cite{amato2009achieving}, and in a new \textsc{MaritimeMEDEVAC} benchmark. Problem descriptions with illustrations are disclosed in the technical appendix. All experiments were conducted using an AMD Ryzen 9 9900X3D 12-Core Processor (4400 MHz), with timeout occurring at 20 minutes and memory out at 16 GB. We adopt hyper-parameters $M$ = 200, $d$ = 3, and $\alpha$ = 0.2 for all experiments. A link to our code repository is provided in the technical appendix to support reproducibility.  

\begin{figure}[ht]
  \centering
  \includegraphics[width=0.95\linewidth]{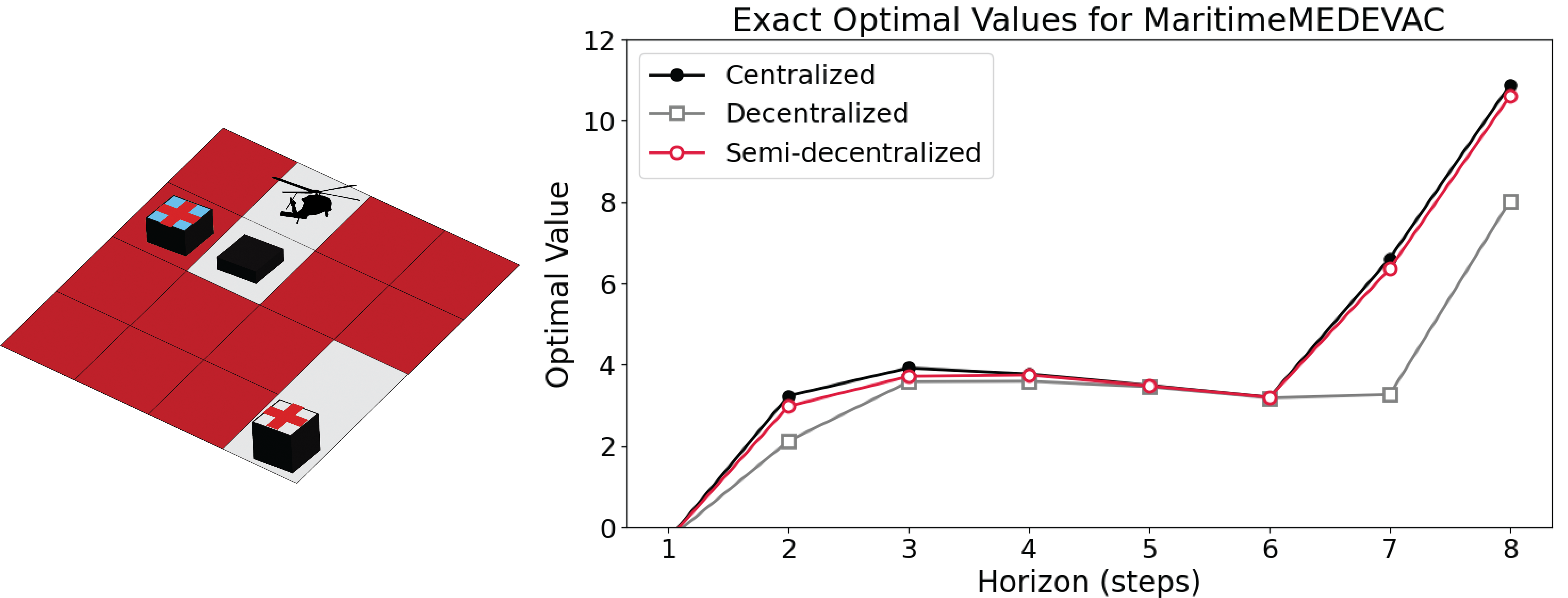}
  \caption{\textsc{MaritimeMEDEVAC} environment representation and centralized/decentralized/semi-decentralized optimal policy values for horizons one through eight.}
  \label{fig:medevac_optimal}
\end{figure}

\noindent \textbf{Results} As shown in Table 2 and Figure \ref{fig:medevac_optimal}, RS-SDA* is competitive with the centralized upper bound across most semi-decentralized benchmarks and \textsc{MaritimeMEDEVAC}. The modified benchmarks demonstrate how model dynamics influence the value of information in multi-agent systems. \textsc{SDec-FireFighting} exemplifies problems where centralization benefits are negligible, and the optimal RS-SDA* solution equals the optimal RS-MAA* solution for all considered $h$. By contrast, \textsc{SDec-Box} exemplifies problems where partial centralization results in complete information sharing, and the optimal RS-SDA* solution equals the fully centralized optimum for all considered $h$. In \textsc{MaritimeMEDEVAC}, the three regimes are nearly indistinguishable at moderate horizons ($H=4,5,6$), but at $H=7$ centralized reaches $6.62$ while semi-decentralized attains $6.36$, clearly outperforming full decentralization ($3.27$). At $H=7$, the semi-decentralized policy recovers about $96\%$ of the centralized value. These results indicate that semi-decentralization and RS-SDA* preserve much of the benefit of centralized coordination while staying tractable, with occasional slowdowns or timeouts on problem instances where lossless clustering is largely ineffective. 

\begin{table}[htbp]
\centering
\resizebox{\columnwidth}{!}{%
\begin{tabular}{l|lr|lr|lr}
\toprule
\multicolumn{1}{c}{} & \multicolumn{2}{c}{lower bound} & \multicolumn{2}{c}{\textbf{our approach}}  & \multicolumn{2}{c}{upper bound} \\
\midrule
\multicolumn{1}{c}{} & \multicolumn{2}{c}{\textit{decentralized}} & \multicolumn{2}{c}{\textit{semi-decentralized}}  & \multicolumn{2}{c}{\textit{centralized}} \\
\multicolumn{1}{c}{} & \multicolumn{2}{c}{RS-MAA*} & \multicolumn{2}{c}{\textbf{RS-SDA*}}  & \multicolumn{2}{c}{RS-SDA*} \\
$h$ & value & time & value & time & value & time \\
\midrule
\multicolumn{7}{l}{\cellcolor{gray!15}\textbf{\textsc{SDec-Tiger}}} \\
8 & 12.21726 & 1 & 27.21518 & <1 & 47.71696 & <1 \\
9 & 15.57244 & 6 & 30.90457 & <1 & 53.47353 & <1 \\
10 & 15.18438 & 271 & 34.72370 & <1 & 60.50988 & <1 \\
\midrule
\multicolumn{7}{l}{\cellcolor{gray!15}\textbf{\textsc{SDec-FireFighting} ($n_h = 3, n_f = 3$)}} \\
3 & -5.73697 & <1 &  -5.73697 & <1 & -5.72285 & <1 \\
4 &  -6.57883 & 2 &  -6.57883 & 3 & -6.51859 & 3 \\
5 & -7.06987 & 29 & -7.06987 & 49 & -6.98069 & 37 \\
\midrule
\multicolumn{7}{l}{\cellcolor{gray!15}\textbf{\textsc{SDec-BoxPushing}}} \\
3 & 66.08100 & <1 & 66.81000 & <1 & 66.81000 &  <1 \\
4 & 98.59361 & 18 & 99.55630 & <1 & 99.55630 &  <1 \\
5 & 107.72985 & MO & 109.09251 & 1 & 109.09251 & 1 \\
\midrule
\multicolumn{7}{l}{\cellcolor{gray!15}\textbf{\textsc{SDec-Mars} (Right Band Rendezvous)}} \\
4 & 10.18080 & 1 & 10.18080 & <1 & 10.87020 & <1 \\
5 & 13.26654 & 2 & 14.29038 & <1 & 14.38556 & 1 \\
6 & 18.62317 & 4 & 20.06430 & 2 & 20.06706 & 3 \\
\midrule
\multicolumn{7}{l}{\cellcolor{gray!15}\textbf{\textsc{SDec-Mars} (Survey Site Beacons)}} \\
4 & 10.18080 & 1 & 10.54620 & <1 & 10.87020 & <1 \\
5 & 13.26654 & 2 & 13.26654 & <1  & 14.38556 & 1 \\
6 & 18.62317 & 4 & 18.62317 & 143 & 20.06706 & 3 \\
\midrule
\multicolumn{7}{l}{\cellcolor{gray!15}\textbf{\textsc{SDec-Mars} (Drill Site Beacons)}} \\
4 & 10.18080 & 1 & 10.87020 & <1 & 10.87020 & <1 \\
5 & 13.26654 & 2 & 14.38556 & <1 & 14.38556 & 1 \\
6 & 18.62317 & 4 & 20.06168 & 2 & 20.06706 & 3 \\
\midrule
\multicolumn{7}{l}{\cellcolor{gray!15}\textbf{\textsc{MaritimeMEDEVAC}}} \\
6 & 3.18348 & <1 & 3.19807 & 28 & 3.19945 & <1 \\
7 & 3.26710 & 2 & 6.36301 & 33 & 6.61819 & 1 \\
8 & 8.03228 & 156 & 10.61275 & MO & 10.88244 & 1 \\
\bottomrule
\end{tabular}}
\caption{RS-SDA* offline planning performance on semi-decentralized benchmarks and \textsc{MaritimeMEDEVAC}. TO denotes timeout ($>$1200s) and MO denotes memout ($>$16GB).}
\end{table}

\section{Conclusion}
We present a foundational framework for multiagent decision making under probabilistic communication.  We introduce the SDec‑ POMDP, which unifies the Dec‑POMDP, MPOMDP, and several communication mechanisms with delay, loss, or cost. A secondary contribution is RS‑SDA*, an admissible heuristic search algorithm for semi-decentralized systems with excellent performance, and semi-decentralized versions of four standard benchmarks and a new medical evacuation scenario. Taken together, SDec‑POMDP and RS‑SDA* provide a principled basis for studying and exploiting probabilistic communication in cooperative teams. Future work includes exploiting interleaving offline planning and online search to improve approximate RS-SDA* performance and investigating systems with non-stationary sojourn time distributions. 
\clearpage

\bibliographystyle{ACM-Reference-Format} 
\bibliography{sample}

\clearpage


\appendix 

\section{Appendix}

\subsection{k-Steps Delayed Communication} 

\textit{$k$-steps delayed communication} \cite{nayyar2010optimal}, \cite{oliehoek2013sufficient}, \cite{oliehoek2008optimal} is a model for delayed broadcast communication where each agent receives the complete joint history from $t-k$ at $t$. This enables $k$-steps delayed communication to generate a joint belief $b^{t-k}$ on which to conduct subsequent decentralized planning. We formally define the agent histories under $k$-steps delayed communication, below: 

\begin{equation*}
    H^t_i = (A_i O_i)^t
\end{equation*}

\begin{equation*}
    \bar{h}_i \in \bar{H}_{i, k} = \bigcup_{t = 1}^\infty \prod_{j \in I} 
    \begin{cases}
        H_i^t & \text{if } j=i \\
        H_j^{\text{max\{0, $t-k$\}}} & \text{otherwise}
    \end{cases}
\end{equation*}

\begin{equation*}
    \bar{\bar{h}} = \prod_{i \in I} \bar{h}_i
\end{equation*}

\begin{lemma}
    SDec-POMDP and $k$-steps delayed communication models are equivalent. 
\end{lemma}

\begin{proof} Demonstrate that 1. $X_\text{$k$-steps Delayed}$ $\leq$ $X_\text{SDec-POMDP}$ and 2. $X_\text{SDec-POMDP} \leq X_\text{$k$-steps Delayed}$ \\

\noindent \textbf{1.} $X_\text{$k$-steps Delayed}$ $\leq$ $X_\text{SDec-POMDP}$

\noindent 
Let $I' = I$ , $S' = S$, $\bar{A}' = \bar{A}$, $R' = R$, and $\bar{\mathcal{O}}' = \bar{\mathcal{O}}$. The state transition and communication sojourn time functions are independent of $\bar{\tau}$ such that $T'(s' \mid s, \bar{a}, \bar{\tau}) =T(s' \mid s, \bar{a})$ and $O'(\bar{o}' \mid s', \bar{a}, \bar{\tau}')$ = $O(\bar{o}' \mid s', \bar{a})$. Again assume a deterministic communication sojourn time function $F$, where $\tau'$ for each agent is fixed to one, resulting in complete centralization at subsequent decision epochs; $\forall i$, $F'(\tau_i' < 1  \mid s, a_i, \tau_i$) = 0 and $F'(\tau_i' = 1  \mid s, a_i, \tau_i$) = 1.  Blackboard selector functions return joint actions and observations at each time-step; $g'(a_{c}^{\text{sel}}) = \bar{a}$ and $h'(o_{c}^{\text{sel}}) = \bar{o}$. Agent selector functions return agent actions and observations at each time-step; $\forall i$, $g'(a_i^{\text{sel}}) = a_i$ and $\forall i$, $h'(o_i^{\text{sel}}) = o_i$. \\

\noindent \textbf{2.} $X_\text{SDec-POMDP}$ $\leq$ $X_\text{$k$-steps Delayed}$ \\
\noindent Let $k=0$. See proof of Lemma 4 for construction of an SDec-POMDP model within a Dec-POMDP. 
\end{proof}

\begin{lemma}
SDec-POMDP and $k$-steps delayed model-policy structures are equivalent.  
\end{lemma}

\begin{proof}Demonstrate that $XY_\text{$k$-steps Delayed}$ $\leq$ $XY_\text{SDec-POMDP}$ and 2. $XY_\text{SDec-POMDP} \leq XY_\text{$k$-steps Delayed}$ \\

\noindent \textbf{1.} $XY_\text{$k$-steps Delayed}$ $\leq$ $XY_\text{SDec-POMDP}$

\noindent For any object X carrying a time index, $X\big|_{0:r}$ denotes the same object with indices $> r$ disabled/ignored. Let each agent memory selector function return the blackboard memory at $t-k$, such that $\forall i$, $f'(m_{ci}^\text{sel}) = m_c^t\big|_{0:\,t-k}$. The blackboard memory selector function returns the latest joint agent memory set. The update rule for each agent's memory is deterministic and concatenates the prior memory with the set of individual agent actions and observations from the current time-step: 

\begin{equation*}
    \eta'(m_i') = \begin{cases}1, \text{if } m_i' = m_i m_{ci}^\text{sel} a_i^\text{sel} o_i^\text{sel} \\
        0, \text{otherwise.}
        \end{cases}
\end{equation*}

\noindent Agent actions are selected using a policy over the agent's memory:

\begin{equation*}
    \psi'(\bar{a}) = 
    \begin{cases} 1, \forall i, 
        a_i = \pi_i(m_i) \\
        0, \text{otherwise.}
    \end{cases}
\end{equation*}

\noindent Finally, the update rule for the blackboard memory is deterministic and concatenates the prior shared memory with the complete set of agent actions and observations from the current time-step: 

\begin{equation*}
\eta_c(m_c') =     
    \begin{cases}1, \text{if } m_c' = m_c \bar{a}^\text{sel} \bar{o}^\text{sel} \\
    0, \text{otherwise.}
    \end{cases}
\end{equation*}

\noindent \textbf{2.} $XY_\text{SDec-POMDP} \leq XY_\text{$k$-steps Delayed}$ \\
\noindent Let $k=0$. See proof of Lemma 5 for construction of an SDec-POMDP model-policy structure within a Dec-POMDP. 
    
\end{proof}

\begin{lemma}
    SDec-POMDP and $k$-steps delayed model-policy-objective structures are equivalent.
\end{lemma} 

\noindent The semi-decentralized value function reduces to the k-steps delayed value function under the original policy set, beginning with:

\begin{proof} Show $\forall \bar{h}$ $V_{XY_\text{$k$-steps Delayed}}(\bar{h}, b_0) = V_{XY_\text{SDec-POMDP}}(\bar{h}, b_0)$

\noindent $ \begin{aligned}
&  V^{\pi_c, \bar{\pi}}(m_c, \bar{m}, s) = \underbrace{\sum\limits_{\bar{a} \in \bar{A}} \psi(\bar{a} \mid m_c, \bar{m})}_{\textcolor{red}{1, m_c = \bar{h}}} \Biggl[
R(s, \bar{a}) + \gamma \underbrace{\sum\limits_{\bar{\tau} \in \bar{T}} F(d\bar{\tau} \mid s, \bar{a})}_{\textcolor{red}{1, \text{ all RV } \ind \text{ of } d\tau}} \\
&  
\sum\limits_{s' \in S} T(s' \mid s, \bar{a}) 
\sum\limits_{\bar{o}' \in \bar{\mathcal{O}}} O(\bar{o}' \mid s', \bar{a}) \\ 
& \underbrace{\sum\limits_{m_c' \in M_c} \eta_c (m_c' \mid m_c, f(\bar{m}), g(\bar{a}), h(\bar{o}'))}_{\textcolor{red}{1, m_c' = m_c \bar{a}^\text{sel} \bar{o}^\text{sel}}} \\
& \underbrace{\sum\limits_{\bar{m}' \in \bar{M}} \prod\limits_{i \in I} \eta(m_i' \mid m_i, f(m_c), g(\bar{a}), h(\bar{o}'))}_{\textcolor{red}{1, \forall i, m_i' = m_i a_i^\text{sel} o_i^\text{sel}}} V^{\pi_c, \bar{\pi}}(m_c', \bar{m}', s') \Biggr] \end{aligned}$ \\

\noindent By construction, each under-braced term evaluates deterministically: the blackboard’s memory update enforces $m_c' = m_c \bar{a}^\text{sel} \bar{o}^\text{sel}$, each agent's local memory update yields $m_i' = m_i a_i^\text{sel} o_i^\text{sel}$, and the distribution over communication sojourn times collapse to one. Moreover, since $\psi(\bar{a} \mid m_c, \bar{m}, s) = 1$ whenever $\bar{a} = \bar{\pi}(\bar{\bar{h}})$, we obtain $V^{\pi_c, \bar{\pi}}(m_c, \bar{m}, s) = V^{\bar{\pi}, k}(s, \bar{\bar{h}})$, which results in standard value recursion: 
\begin{equation*}
R(s, \bar{\pi}(\bar{\bar{h}})) + \gamma \sum\limits_{s' \in S}  \sum\limits_{\bar{o}' \in \mathcal{O}} \text{Pr}(s', \bar{o}' \mid s, \bar{\pi}(\bar{\bar{h}})) V^{\bar{\pi}}(s', \bar{\bar{h}}) 
\end{equation*}

\noindent
Observe that:

\begin{equation*}\sum\limits_{s' \in S} T(s' \mid s, \bar{a}) \sum\limits_{\bar{o}' \in \bar{\mathcal{O}}} O(\bar{o}' \mid s', \bar{a}) =  \sum\limits_{s' \in S}  \sum\limits_{\bar{o}' \in \mathcal{O}} \text{Pr}(s', \bar{o}' \mid s, \bar{\pi}(\bar{\bar{h}}))
\end{equation*}
    
\end{proof} 

\begin{proposition}
    The SDec-POMDP and $k$-steps delayed are equivalent.
\end{proposition}

\subsection{Dec-POMDP-Com}

The \textit{Dec-POMDP-Com} \cite{goldman2003optimizing} extends explicit communication to the Dec-POMDP by including an alphabet of possible messages $\Sigma$ and communication cost function $C_\Sigma$. For a specified cost, each agent takes a communication action after their control action, which under the instantaneous broadcast communication assumption results in all other agents receiving an additional observation. Unlike the SDec-POMDP, agents in a Dec-POMDP-Com are never entirely restricted from communication. The algorithm designer may centralize the agents in a Dec-POMDP-Com (for cost) at will.

\begin{lemma}
    SDec-POMDP and Dec-POMDP-Com models are equivalent.
\end{lemma}

\begin{proof}Demonstrate that $X_\text{Dec-POMDP-Com}$ $\leq$ $X_\text{SDec-POMDP}$ and 2. $X_\text{SDec-POMDP} \leq X_\text{Dec-POMDP-Com}$ \\

\noindent \textbf{1.} $X_\text{Dec-POMDP-Com}$ $\leq$ $X_\text{SDec-POMDP}$ \\
\noindent Reference $X_\text{Dec-POMDP-Com}$ $\leq_p$ $X_\text{Dec-POMDP}$ \cite{seuken2008formal} and $X_\text{Dec-POMDP}$ $\leq$ $X_\text{SDec-POMDP}$ in proof of Lemma 4. \\

\noindent \textbf{2.} $X_\text{SDec-POMDP}$ $\leq$ $X_\text{Dec-POMDP-Com}$ \\
\noindent Reference $X_\text{SDec-POMDP}$ $\leq$ $X_\text{Dec-POMDP}$ in proof of Lemma 4 and $X_\text{Dec-POMDP}$ $\leq_p$ $X_\text{Dec-POMDP-Com}$ \cite{seuken2008formal}. 

\end{proof}

\begin{lemma}
    SDec-POMDP and Dec-POMDP-Com model-policy structures are equivalent. 
\end{lemma} 

\begin{proof} Demonstrate that $XY_\text{Dec-POMDP-Com}$ $\leq$ $XY_\text{SDec-POMDP}$ and 2. $XY_\text{SDec-POMDP} \leq XY_\text{Dec-POMDP-Com}$ \\

\noindent \textbf{1.} $XY_\text{Dec-POMDP-Com}$ $\leq$ $XY_\text{SDec-POMDP}$ \\
\noindent Reference $XY_\text{Dec-POMDP-Com}$ $\leq_p$ $XY_\text{Dec-POMDP}$ \cite{seuken2008formal} and $XY_\text{Dec-POMDP}$ $\leq$ $XY_\text{SDec-POMDP}$ in proof of Lemma 5. \\

\noindent \textbf{2.} $XY_\text{SDec-POMDP} \leq XY_\text{Dec-POMDP-Com}$ \\
\noindent Reference $XY_\text{SDec-POMDP}$ $\leq$ $XY_\text{Dec-POMDP}$ in proof of Lemma 5 and $XY_\text{Dec-POMDP}$ $\leq_p$ $XY_\text{Dec-POMDP-Com}$ \cite{seuken2008formal}. 

\end{proof}

\begin{lemma}
    SDec-POMDP and Dec-POMDP-Com model-policy-objective structures are equivalent. 
\end{lemma}

\begin{proof} Demonstrate that $\forall \bar{h}$ $V_{XY_\text{Dec-POMDP-Com}}(\bar{h}, b_0) = $ \\ $V_{XY_\text{SDec-POMDP}}(\bar{h}, b_0)$ \\

\noindent Reference $\forall \bar{h}$ $V_{XY_\text{Dec-POMDP}}(\bar{h}, b_0) = V_{XY_\text{SDec-POMDP}}(\bar{h}, b_0)$ in proof of Lemma 6. 

\end{proof}

\begin{proposition}
    The SDec-POMDP and Dec-POMDP-Com are equivalent.
\end{proposition}

\subsection{Semi-Decentralized Benchmarks}
\subsubsection{\textsc{SDec-Tiger}}

Consider the following semi-decentralized variation on the \textsc{Dec-Tiger} benchmark \cite{nair2003taming}, depicted in Figure \ref{fig: SDecTiger}. \textsc{SDec-Tiger} has 2 states, 3 actions, and 2 observations. Two cooperative agents stand behind two doors. One door leads to a room containing a tiger while the other leads to a room containing treasure. Each agent has three actions: opening the left door $OL$, opening the right door $OR$, and listening $L$. The problem reward function is fully described in table \ref{table:tiger_rewards}. The problem resets when any door is opened; the probability that the tiger is behind the left door $TL$ and that the tiger is behind the right door $TR$ both become 0.5. Joint actions that do not open doors do not affect the underlying state. Agents have a 75\% chance of accurately communicating their action observation histories if they both listen. After taking an action, each agent receives one of two observations, hear tiger on left $HL$ or hear tiger on right $HR$. Listening to either door gives an  0.75 probability of returning the correct observation. Opening a door and resetting the problem results in both agents receiving either observation with a 0.5 probability. If one agent opens a door while the other listens, the listening agent will not know the problem has been reset.

\begin{table}[h!]
\footnotesize
\centering
\caption{\textsc{SDec-Tiger} Rewards ($TL$, $TR$)}
\label{table:tiger_rewards}
\begin{tabular}{l|lll} \toprule
 $\bar{a}$ & $OL$ & $OR$ & $L$ \\
 \midrule
 $OL$ & (-50, 20) & (-100, -100) & (-101, 9) \\ 
 $OR$ & (-100, -100) & (20, -50) & (9, -101) \\ 
 $L$ & (-101, 9) & (9, -101) & (-2, -2)\\
\bottomrule
\end{tabular}
\end{table}

\begin{figure}[ht!]
  \centering
  \includegraphics[width=1.0\linewidth]{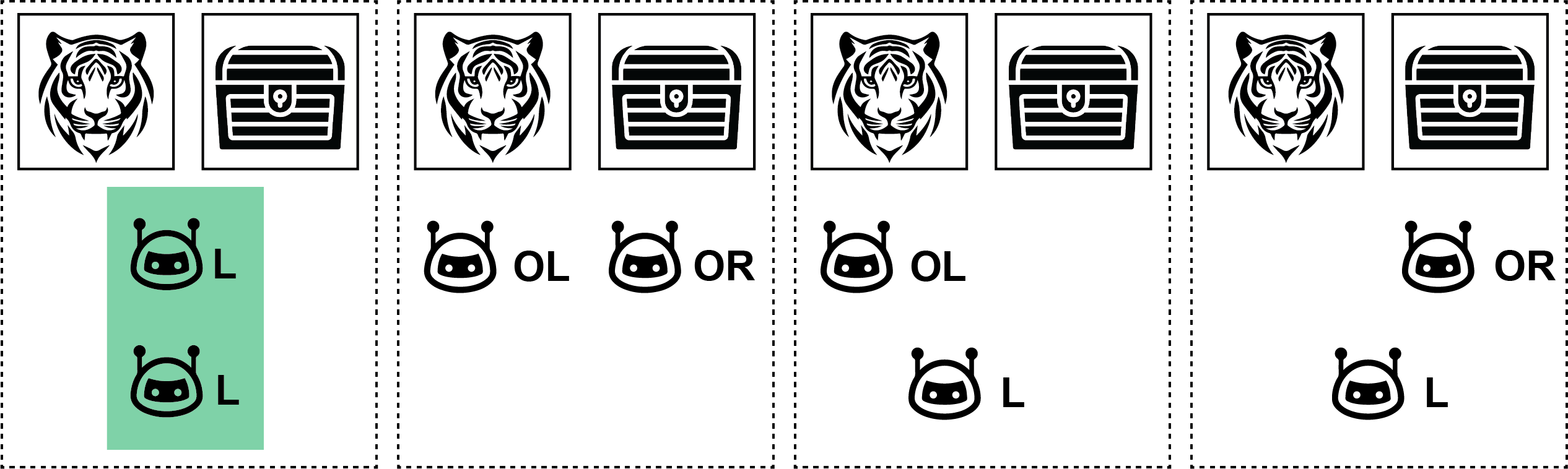}
  \caption{Illustration of four of nine possible joint actions for \textsc{SDec-Tiger}. Agents communicate their observation histories with some probability when they listen to the same door (in green).}
  \label{fig: SDecTiger}
\end{figure}

\subsubsection{\textsc{SDec-FireFighting}}

Consider the following semi-decentralized variation on the \textsc{FireFighting} benchmark \cite{oliehoek2008optimal}, depicted in Figure \ref{fig: SDecFireFighting}. \textsc{SDecFireFighting} ($n_f = 3, n_h=3$) has 432 states, 3 actions, and 2 observations. 2 agents are tasked with addressing a line of $n_h$ houses, each with fire severity status $f$ in range [$0, n_f$] initially sampled from a uniform distribution. Each agent selects a house to suppress at each time-step. Single agent suppression decrements $f$ by 1 with probability 1.0 if all adjacent houses have $f=0$ or with probability 0.6 otherwise. Dual agent suppression resets $f$ to $0$. Agents only communicate their observation histories if they suppress the same house. A house without a firefighter present increments its $f$ by 1 with probability 0.8 if an adjacent house has $f>0$ or with probability 0.4 if all adjacent houses have $f=0$. A house with $f=0$ will catch fire (increment $f$ by 1) with probability 0.8 if an adjacent house has $f>0$. Each agent observes their selected house to be on fire or not with probability 0.2 if $f=0$, probability 0.5 if $f=1$, and probability 0.8 if $f\geq 2$. The cooperative agent team is rewarded the summation of $-f$ across all $n_h$ following action selection. 

\begin{figure}[ht!]
  \centering
  \includegraphics[width=0.6\linewidth]{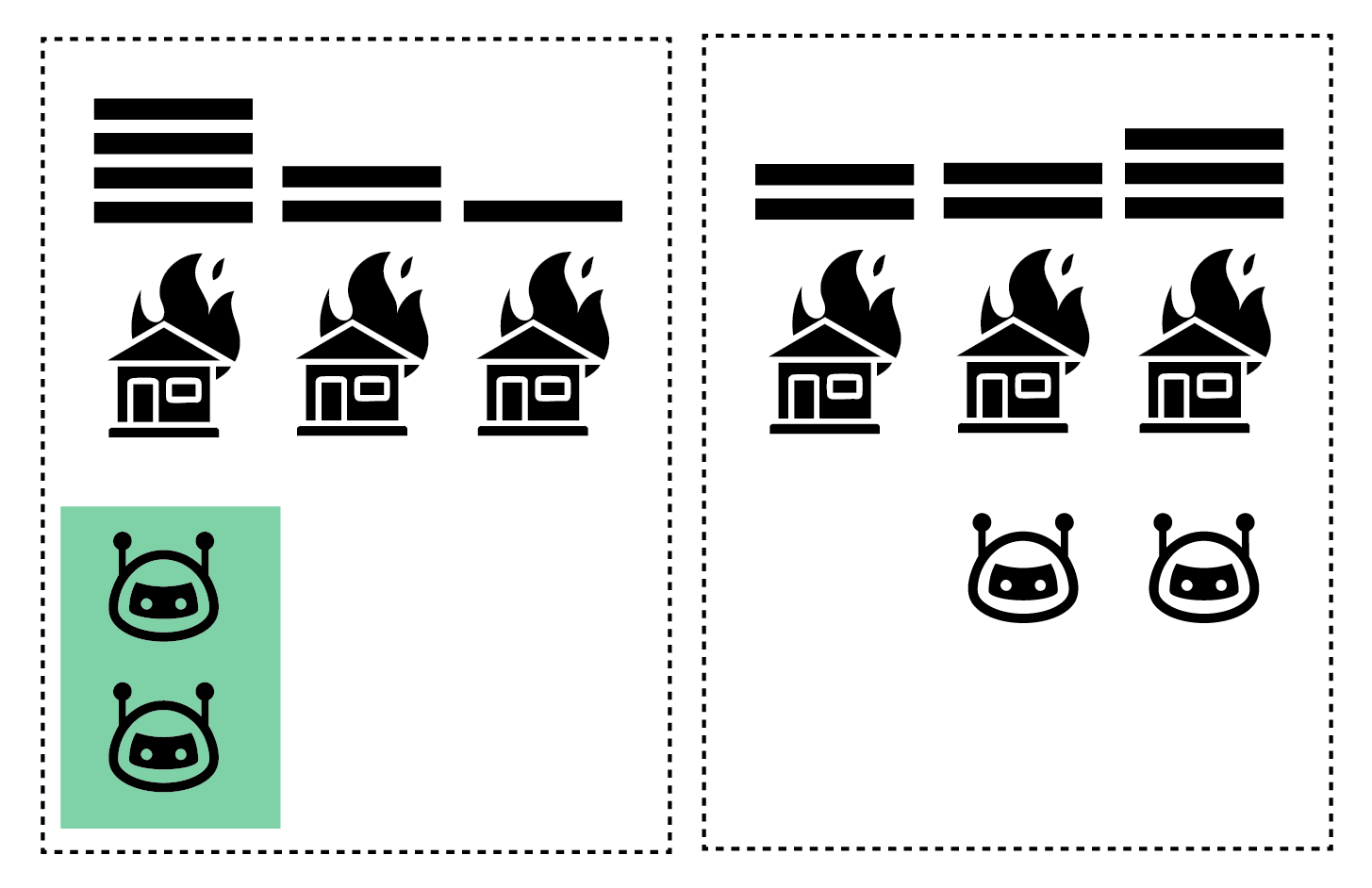}
  \caption{Illustration of two of nine joint actions in \textsc{SDec-FireFighting} ($n_h = 3, n_f = 4$). Agents communicate when they suppress the same house, shown in green.}
  \label{fig: SDecFireFighting}
\end{figure}

\subsubsection{\textsc{SDec-BoxPushing}}

Consider the following semi-decentralized variation on the \textsc{BoxPushing} benchmark \cite{seuken2007improved}, depicted in Figure \ref{fig: SDecBoxPushing}. \textsc{SDecBoxPushing} has 100 states, 4 actions, and 5 observations. $n$ agents cooperate to push small and large boxes into an established goal area. Each agent can choose to rotate left, rotate right, move forward, or remain in place. Rotation and movement actions are successful with a 0.9 probability, otherwise the agent remains in place. Forward movement while facing a box will cause the box to translate one unit in the direction of movement, if permissible. A single agent can push a small box, but two agents must act in tandem to push a large box. Movement into a wall, or into a large box with one agent, will result in remaining in place. Each agent correctly observes what is in front of them: a wall, a small box, a large box, an empty space, or another agent. Agents share their observation histories when simultaneously occupying one or more established communication grid squares. Agents receive a $-0.1n$ reward after each time-step, a $-5$ reward for each agent that moves into a wall, a $+10$ reward for each small box pushed into the goal area, and a $+100$ reward for each large box pushed into the goal area. The problem resets as soon as any box reaches the goal state. We adopt the environment configuration depicted in Figure \ref{fig: SDecBoxPushing}. 

\begin{figure}[ht!]
  \centering
  \includegraphics[width=0.4\linewidth]{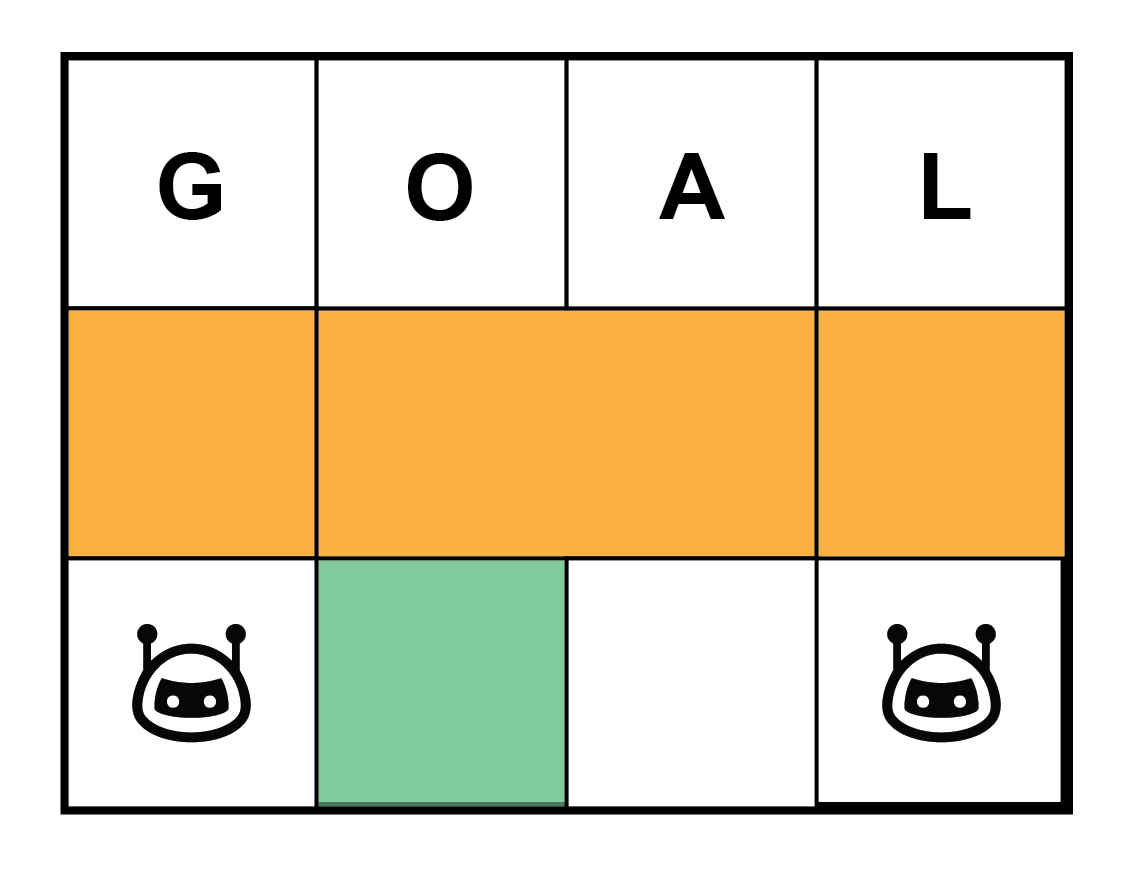}
  \caption{Illustration of the \textsc{SDec-BoxPushing} environment. Agents communicate when they are both in the green square.}
  \label{fig: SDecBoxPushing}
\end{figure}

\subsubsection{\textsc{SDec-Mars}}

Consider the following semi-decentralized variation on the \textsc{Mars} benchmark \cite{amato2009achieving}, depicted in Figure 8. \textsc{SDec-Mars} has 256 states, 6 actions, and 8 observations. Each of two agents can choose to move north, south, east, and west in a 2x2 grid, or conduct an experiment of choice (drilling or sampling) in their current location. Two grid squares are intended to be sampled by one agent, and the other two grid squares require that both agents drill simultaneously. Each agent accurately observes their location in the 2x2 grid and whether an experiment has already been performed there. Agents share their observation histories while simultaneously occupying a designated communication grid square. The problem resets once an experiment is performed in all four grid squares. The cooperative agent team receives a large positive reward for drilling a drill site, a small positive reward for sampling a sample site, a large negative reward for drilling a sample site, and a small positive reward for sampling a drill site. Attempting a second experiment on the same site incurs a small negative reward. 

\begin{figure}[ht!]
  \begin{subfigure}{0.30\linewidth}
    \includegraphics[width=\linewidth]{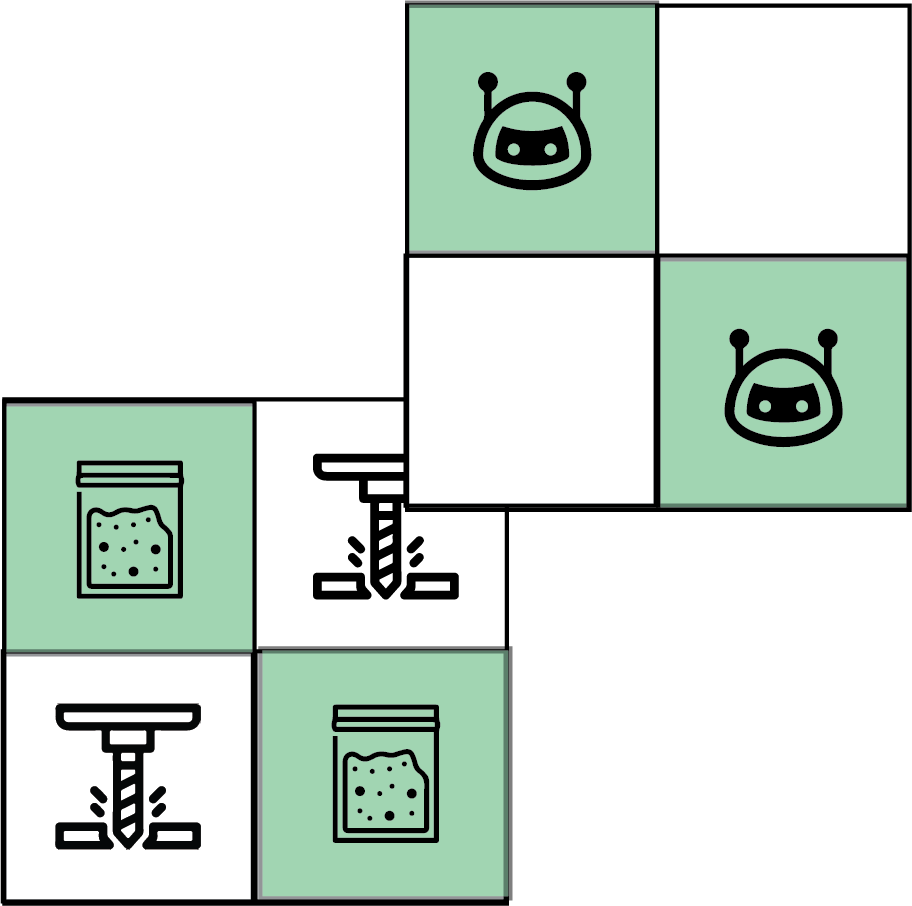}
    \caption{}
    \label{fig:SDM1}
  \end{subfigure}%
  \hfill
  \begin{subfigure}{0.30\linewidth}
    \includegraphics[width=\linewidth]{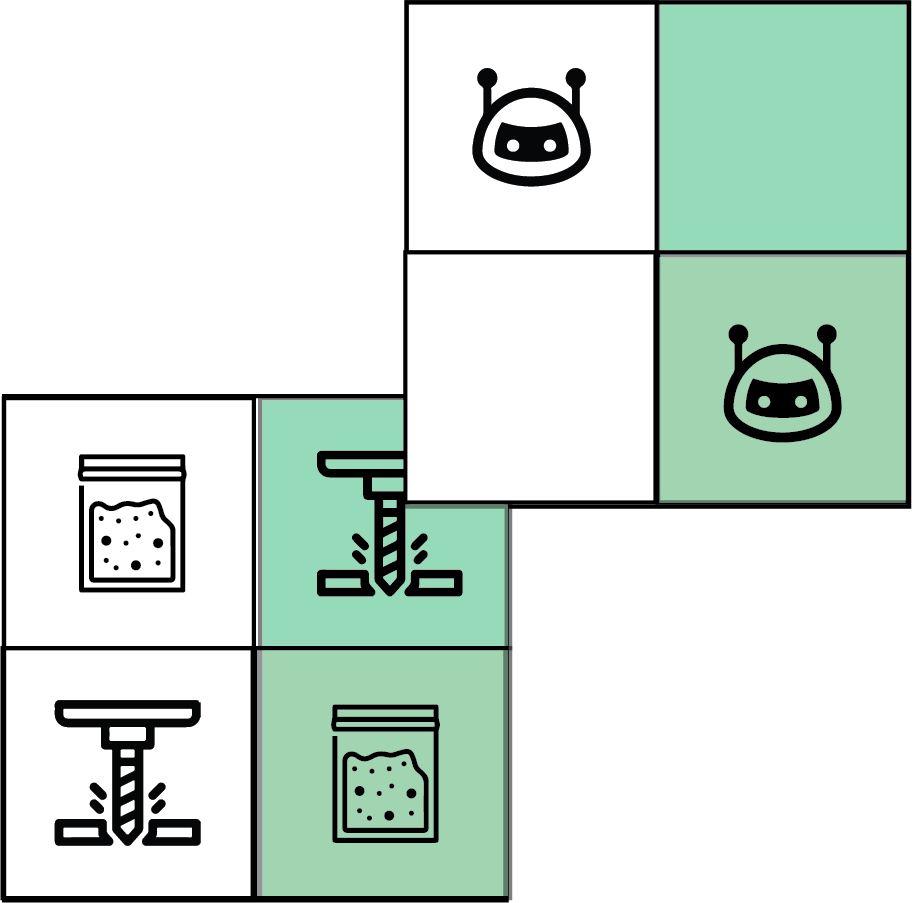}
    \caption{}
    \label{fig:SDM2}
  \end{subfigure}%
  \hfill
  \begin{subfigure}{0.30\linewidth}
    \includegraphics[width=\linewidth]{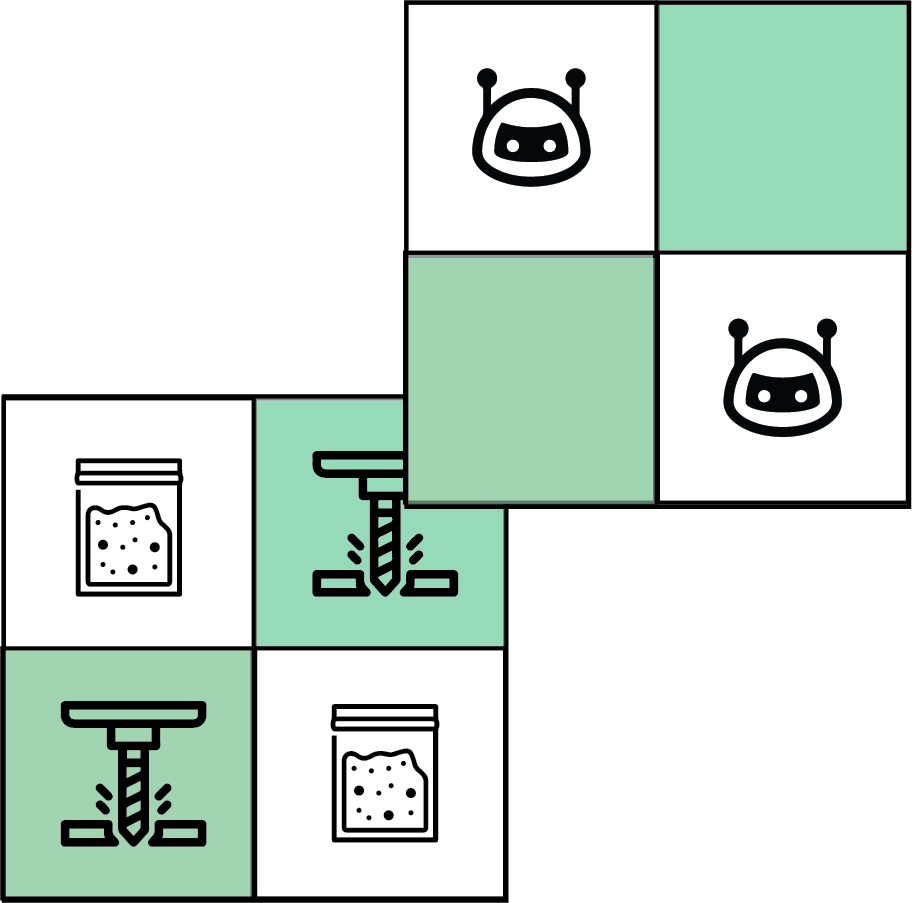}
    \caption{}
    \label{fig:SDM3}
  \end{subfigure}
  \label{fig: SDecMars}
  \caption{Illustration of the \textsc{SDec-Mars} environment. Agents in (a), the survey site beacon scenario, can communicate when co-located or adjacent to the same not yet surveyed survey site. Agents in (b), the right band rendezvous scenario, communicate alongside the right side of the grid when at least one site remains incomplete. Agents in (c), the drill site beacon scenario, can communicate when co-located or adjacent to the same not yet drilled drill site.}
\end{figure}

\begin{figure}[ht!]
  \centering
  \includegraphics[width=0.4\linewidth]{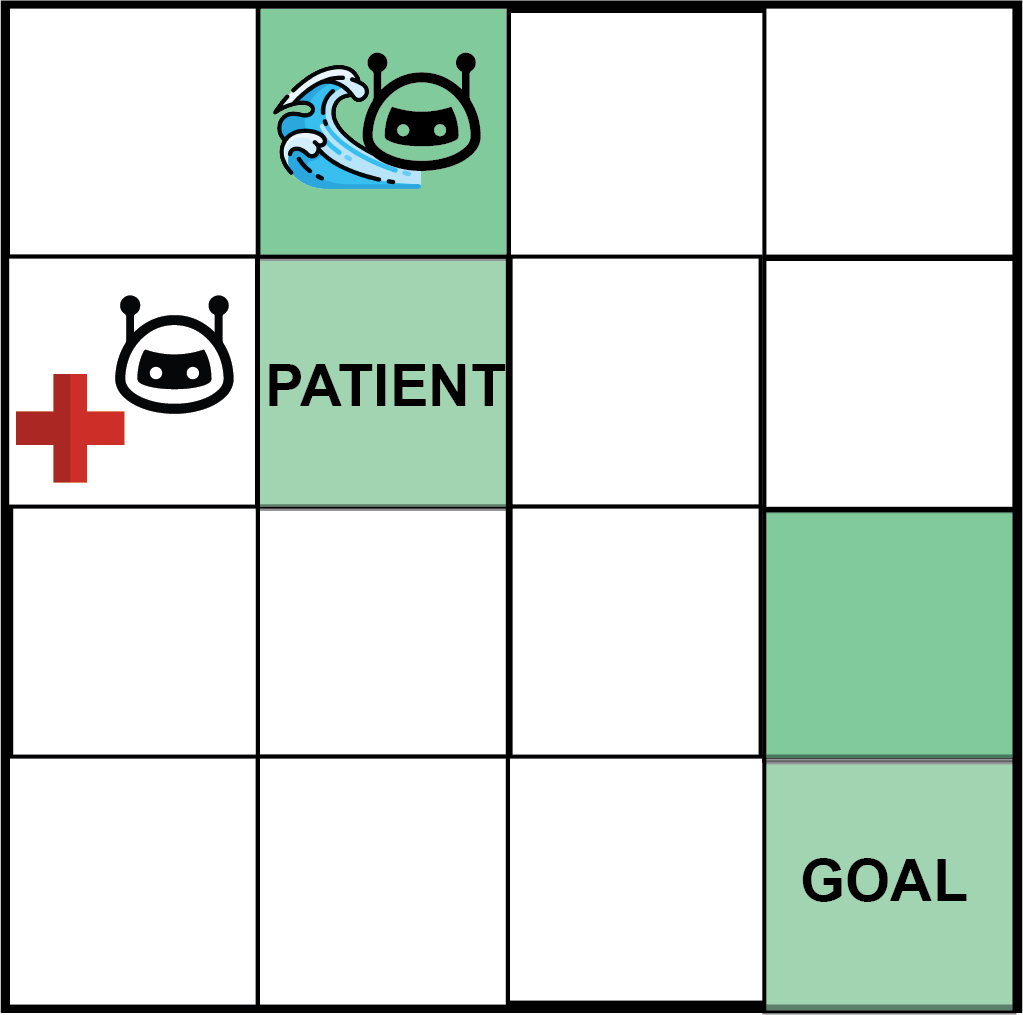}
  \caption{Illustration of the \textsc{MaritimeMEDEVAC} environment. Agents communicate when they are positioned adjacent to patient pickup and drop-off sites.}
  \label{fig: MaritimeMEDEVAC}
\end{figure}

\subsubsection{\textsc{MaritimeMEDEVAC}}
We introduce a new semi-decentralized MEDEVAC benchmark involving a $4 \times 4$ gridworld archipelago, depicted in Figure \ref{fig: MaritimeMEDEVAC}. \textsc{MaritimeMEDEVAC} has 512 states, 3 actions, and 2 observations. Two agents, a medical aircraft and a transport ship, must retrieve a patient at $(1,1)$ and deliver them to a hospital at $(3,3)$. At each time-step, each agent selects one of ${\textsc{Wait}, \textsc{Advance}, \textsc{Exchange}}$. \textsc{Advance} moves one cell toward the current target (patient if $carry = 0$, else hospital), succeeding independently with probability 0.95 for the aircraft and 0.85 for the boat. $\textsc{Wait}$ leaves the agent position unchanged. \textsc{Exchange} attempts a \textit{joint} pickup/drop that succeeds with probability 0.95 when both agents are at the corresponding site (toggling $carry$). Each agent receives a binary observation indicating whether it is \textit{at-target} (patient if $carry = 0$, hospital if $carry = 1$) or \textit{not}. The team incurs -0.3 per step, issuing \textsc{Exchange} away from $\{(1,1), (3,3)\}$ incurs -1.0, a solo pickup or solo drop-off incurs -6.0, and joint pickup or drop-off grants +5.0 and +12.0 respectively. Agents share observation histories in a subset of ``one-arrived, one-not'' states: at the \textit{patient} when the aircraft is at $(1,1)$ and the boat remains at $(1,0)$ with $(carry=0)$, and at the hospital when one agent is at $(3,3)$ and the other at $(3,2)$ with $(carry=1)$ (both permutations). Agents cannot communicate in any other states.

\subsection{Code}
Results for semi-decentralized benchmarks may be reproduced at: https://github.com/mahdial-husseini/RSSDA

\subsection{Notation}

\begingroup
\renewcommand{\arraystretch}{1.5} 
\begin{tabular}{ll}
$\tau$ & sojourn time (assignment of random variable $\mathcal{T}$) \\
$\tau^t$ & sojourn time corresponding to decision epoch $t$\\
$\eta$ & natural process time \\
$\eta^t$ & decision epoch $t$ start time within natural process \\
$Q$ & complete state transition distribution \\
$F$ & sojourn time distribution \\
$f$ & selector for propagating agent memory information \\
$g$ & selector for propagating agent action information  \\
$h$ & selector for propagating agent observation information  \\
$m_{ci}^{sel}$ & memory propagated to the blackboard \\
$m_{i}^{sel}$ & memory propagated to agent $i$ \\
$a_i^{\text{sel}}$ & action information propagated to agent $i$ \\
$a_c^{\text{sel}}$ & action information propagated to blackboard \\
$o_i^{\text{sel}}$ & observation information propagated to agent $i$ \\
$o_c^{\text{sel}}$ & observation information propagated to blackboard \\
$m_c$ & blackboard memory \\
$m_i$ & agent $i$'s memory \\
$\psi$ & distribution over action selection \\
$\eta$ & distribution over memory updates (overloaded) \\
$Z_{i}^\text{sel}$ & $\langle M_{ci}^\text{sel}, \bar{A}_{i}^\text{sel}, \bar{O}_{i}^\text{sel}\rangle$ \\ 
$Z_{c}^\text{sel}$ & $\langle \bar{M}^\text{sel}, \bar{A}^\text{sel}, \bar{O}^\text{sel}\rangle$
\end{tabular}
\endgroup

\end{document}